\newif\ifextras
\extrasfalse
\newif\iflong
\longtrue
\def\final{1}  

\newif\ifshort

\iflong
\shortfalse
\else
\shorttrue
\fi

\newif\ifcolt
\coltfalse
\ifcolt
\documentclass[10pt, conference, onecolumn, compsocconf]{IEEEtran}
\else
\documentclass[11pt]{article}\newcommand{\numberofauthors}[1]{}  \usepackage{amsthm} \newcommand{\email}[1]{\href{mailto:#1}{\texttt{#1}}} \newcommand{\myand}{\and} \newcommand{\alignauthor}{}  \usepackage{fullpage}
\fi

\iflong
\newcommand{\myparagraph}[1]{\emph{#1}}
\else
\newcommand{\myparagraph}[1]{\paragraph{#1}}
\usepackage{enumitem}
\setlist[enumerate]{noitemsep,nolistsep}
\fi

\usepackage{amsmath, amssymb, microtype}
\usepackage{bbm}
\usepackage{latexsym}
\usepackage{color}

\usepackage{algorithm}
\usepackage{algpseudocode}

\usepackage{float}

\floatstyle{ruled}
\newfloat{subroutine}{htbp}{loa}
\floatname{subroutine}{Subroutine}
\algdef{SE}[myFOR]{myFor}{myEndFor}[1]{\algorithmicfor\ #1\ \algorithmicdo}{\algorithmicend\ \algorithmicfor}%

\usepackage{mathtools}
\usepackage{amsfonts}
\usepackage{hyperref}
\usepackage{thmtools}
\usepackage{thm-restate}

\ifnum\final=0  
\newcommand{\lnote}[1]{[{\small Luis: \bf #1}]}
\newcommand{\jnote}[1]{[{\small Joe: \bf #1}]}
\newcommand{\nnote}[1]{[{\small Navin: \bf #1}]}
\newcommand{\anonnote}[1]{[{\small anon: \bf #1}]}
\newcommand{\sidecomment}[1]{\marginpar{\tiny #1}}
\newcommand{\details}[1]{{\color{blue}\ [[#1]] }}
\else 
\newcommand{\lnote}[1]{}
\newcommand{\jnote}[1]{}
\newcommand{\nnote}[1]{}
\newcommand{\anonnote}[1]{}
\newcommand{\sidecomment}[1]{}
\newcommand{\details}[1]{}
\fi  

\ifcolt

\author{\IEEEauthorblockN{Joseph Anderson\IEEEauthorrefmark{1},
Navin Goyal\IEEEauthorrefmark{2},
Anupama Nandi\IEEEauthorrefmark{1}, and 
Luis Rademacher\IEEEauthorrefmark{1}}
\IEEEauthorblockA{\IEEEauthorrefmark{1} Department of Computer Science and Engineering\\
The Ohio State University,
Columbus, Ohio 43210\\ Email: andejose@cse.ohio-state.edu, nandi.10@osu.edu, lrademac@cse.ohio-state.edu}
\IEEEauthorblockA{\IEEEauthorrefmark{2}Microsoft Research\\
Email: navingo@microsoft.com}}

\newcommand{\mykeywords}[1]{\begin{IEEEkeywords}#1\end{IEEEkeywords}}

\else
\numberofauthors{4}
\author{
\alignauthor
Joseph Anderson \thanks{\email{andejose@cse.ohio-state.edu}, Dept. of Computer Science and Engineering, Ohio State University}\label{osu} \\
\myand Navin Goyal  \thanks{\email{navingo@microsoft.com}, Microsoft Research}\label{msr} \\
\myand Anupama Nandi \thanks{\email{nandi.10@osu.edu}, Dept. of Computer Science and Engineering, Ohio State University} \\
\myand Luis Rademacher \thanks{\email{lrademac@cse.ohio-state.edu}, Dept. of Computer Science and Engineering, Ohio State University} \\
}

\newcommand{\mykeywords}[1]{}
\fi

\declaretheorem[name=Theorem]{theorem}
\newtheorem{lemma}[theorem]{Lemma}

\newtheorem{definition}[theorem]{Definition}

\renewcommand{\dim}{{n}}
\newcommand{\an}{an}
\newcommand{\An}{An}

\newcommand{\RR}{\ensuremath{\mathbb{R}}}
\newcommand{\QQ}{\ensuremath{\mathbb{Q}}}
\renewcommand{\S}{\mathbb{S}}

\newcommand{\expectation}{\operatorname{\mathbb{E}}}
\newcommand{\e}{\expectation}
\newcommand{\conv}{\operatorname{conv}}
\newcommand{\suchthat}{\mathrel{:}}
\newcommand{\norm}[1]{{\lVert#1\rVert}}
\newcommand{\diag}{\operatorname{diag}}
\newcommand{\eps}{\epsilon}

\newcommand{\cov}{\operatorname{Cov}}

\newcommand{\poly}{\operatorname{poly}}

\newcommand{\abs}[1]{\lvert#1\rvert}
\newcommand{\lrabs}[1]{\left \lvert #1 \right \rvert}

\newcommand{\noname}[1]{}

\newcommand{\polar}{\circ}

\newcommand{\measure}{\mathbb{P}}

\newcommand{\inner}[2]{\langle{#1},{#2}\rangle}

\newcommand{\E}{\mathbb{E}}
\newcommand{\R}{\mathbb{R}}
\newcommand{\var}{\operatorname{var}}
\newcommand{\Var}{\var}

\newcommand{\cum}{\mathsf{cum}}

\newcommand{\centroid}{\Gamma}

\newcommand{\hX}{\hat{X}}
\newcommand{\hE}{\hat{E}}

\newcommand{\giventhat}{\mid}
\newcommand{\ud}{\mathop{}\!\mathrm{d}}

\usepackage{verbatim}

\iflong
\else
\renewenvironment{proof}{\expandafter\comment}{\expandafter\endcomment}
\fi

\usepackage[capitalise]{cleveref}

\title{Heavy-tailed Independent Component Analysis}

\date{}

\allowdisplaybreaks

\begin{document}
\maketitle

\begin{abstract}
Independent component analysis (ICA) is the problem of efficiently recovering a matrix $A \in \R^{n\times n}$ from i.i.d. observations of $X=AS$ where $S \in \R^n$ is a random vector with mutually independent coordinates.
This problem has been intensively studied, but all existing efficient algorithms with provable guarantees require that the coordinates $S_i$ have finite fourth moments.
We consider the heavy-tailed ICA problem where we do not make this assumption, about the second moment.
This problem also has received considerable attention in the applied literature.
In the present work, we first give a provably efficient algorithm that works under the assumption that for constant $\gamma > 0$, each $S_i$ has finite $(1+\gamma)$-moment, thus substantially weakening the moment requirement condition for the ICA problem to be solvable.
We then give an algorithm that works under the assumption that matrix $A$ has orthogonal columns but requires no moment assumptions.
Our techniques draw ideas from convex geometry and exploit standard properties of the multivariate spherical Gaussian distribution in a novel way.
\end{abstract}
\mykeywords{Independent Component Analysis, heavy-tailed distributions, centroid body}

\section{Introduction}
\details{
	noise robustness, outliers, polytope learning, robust statistics (refer to the Huber book?), mention why we don't use the floating body(?), talk about prewhitening approaches and why they fail here, talk about mixtures of heavy-tailed distributions(?),
	compare with isotropic PCA (which uses a reweighting similar to our Gaussian damping), preprocessing by symmetrization. Clarify what kind of r.v.s are the $s_i$'s: do we assume that they have density? Does the 
	density have to be non-zero everywhere? 

	Confusingly, in some ICA literature the word ``heavy-tailed'' is used to mean distribution with
	positive Kurtosis. This should be clarified right in the beginning and in the abstract too.
	If you do a search for ``heavy-tailed'' the second result is the Wikipedia page on kurtosis.

	Examples and utility of heavy-tailed distributions.
}

The blind source separation problem is the general problem of recovering 
underlying ``source signals'' that have been mixed in some unknown way and are presented to an observer. 
Independent component analysis (ICA) is a popular model for blind source separation where the mixing is performed linearly. 
Formally, if $S$ is an $\dim$-dimensional random vector from an unknown product distribution and $A$ is an invertible linear transformation, one is tasked with recovering the matrix $A$ and the signal $S$, using only access to i.i.d. samples of the transformed signal, namely $X = AS$.
Due to natural ambiguities, the recovery of $A$ is possible only up to 
the signs and permutations of the columns. Moreover, for the recovery to be possible the distributions of 
the random variables $S_i$ must not be a Gaussian distribution. 
ICA has applications in diverse areas such as neuroscience, signal processing, 
statistics, machine learning. There is vast literature on ICA; see, e.g., \cite{Comon94, ICA01, ComonJutten}.

Since the formulation of the ICA model, a large number of algorithms have been devised employing a diverse set of 
techniques.
Many of these existing algorithms break the problem into two phases: first, find a transformation which, when applied to the observed samples, gives a new distribution which is \emph{isotropic}, i.e. a rotation of a (centered) product distribution; second, one typically uses an optimization procedure for a functional applied to the samples, such as the fourth directional moment, to recover the axes (or basis) of this product distribution.

Our focus in this paper will be on efficient algorithms with provable guarantees and finite sample analysis.
To our knowledge, all known efficient algorithms for ICA with provable guarantees require higher moment assumptions such as finiteness of the fourth or higher moments for each component $S_i$.
Some of the most relevant works, e.g. algorithms of \cite{dl95, FJK}, explicitly require the fourth moment to be finite.
Algorithms in \cite{Yeredor, GVX}, which make use of the characteristic function also seem to require at least the fourth moment to be finite: while the characteristic function exists for distributions without moments, the algorithms in these papers use the second or higher derivatives of the (second) characteristic function, and for this to be well-defined one needs the moments of that order to exist.
Furthermore, certain anticoncentration properties of these derivatives are  needed which require that fourth or higher moments exist.

Thus the following question arises: is ICA provably efficiently solvable when the moment condition is weakened so that, say, only the second moment exists, or even when no moments exist? 
By \emph{heavy-tailed ICA} we mean the ICA problem with weak or no moment conditions (the precise
moment conditions will be specified when needed). 

While we consider this problem to be interesting in its own right, it is also of interest in practice
in a range of applications, e.g. \cite{Kidmose01, KidmoseThesis, shereshevski2001super, ChenBickel04, chen2005consistent, sahmoudi2005blind, wang2009ica}. \nnote{perhaps mention the areas in these papers?}
The problem could also be interesting from the perspective of robust statistics because of the following
 informal connection: 
algorithms solving heavy-tailed ICA might work by focusing on samples in a small (but not low-probability) region in order to get reliable statistics
about the data and ignore the long tail. Thus 
if the data for ICA is corrupted by outliers, the outliers are less likely to affect such an algorithm. 
\nnote{revisit}

In this paper, \emph{heavy-tailed} distributions on the real line are 
those for which low order moments are not finite. Specifically, we will be interested in the case when the fourth
or lower order moments are not finite as this is the case that is not covered by previous algorithms. 
We hasten to clarify that in some ICA literature the word heavy-tailed is used with a different and less standard meaning, namely distributions with positive kurtosis; this meaning will not be used in the present work.

\nnote{It seems to me that this sentence does not fit anywhere: perhaps even more troublesome in the ICA context, however, is that if the first moment of the distribution diverges, one cannot directly put the distribution into isotropic position because the ``center'' is no longer defined.}

Heavy-tailed distributions arise in a wide variety of contexts including signal processing and finance;
see \cite{nolan:2015, Rachev2003} for an extensive bibliography.
Some of the prominent examples of heavy-tailed distributions are the Pareto distribution with shape parameter
$\alpha$ which has moments of order less than $\alpha$, the Cauchy distributions, which has moments of order
less than $1$; many more examples can be found on the Wikipedia page for heavy-tailed distributions.
An abundant (and important in applications) supply of heavy-tailed distributions comes from stable distributions;
see, e.g., \cite{nolan:2015}. There is also some theoretical work on learning mixtures of heavy-tailed 
distributions, e.g., \cite{DasguptaHKS05, ChaudhuriR08a}. 

In several applied ICA models with heavy tails it is reasonable to assume that the distributions have finite first moment.
In applications to finance (e.g., \cite{Chen2007594}), heavy tailed distributions are commonly used to model catastrophic but somewhat unlikely scenarios. A standard measures of risk in that literature, the so called conditional value at risk \cite{RockafellarUryasev}, is only finite when the first moment is finite. Therefore, it is reasonable to assume for some of our results that the distributions have finite first moment.

\subsection{Our results}
Our main result is an efficient algorithm that can recover the independent components when each $S_i$ has $1+\gamma$ moments for $\gamma$ a positive constant. The following theorem states more precisely the guarantees of 
our algorithm. The theorem below refers to the algorithm \emph{Fourier PCA} \cite{} which solves ICA under the 
fourth moment assumption. The main reason to use this algorithm is that finite sample guarantees have been proved
for it; we could have plugged in any other algorithm with such guarantee. The theorem below also refers to 
\emph{Gaussian damping}, which is an algorithmic technique we introduce in this paper and will be explained 
shortly. \nnote{We will assume, for 
simplicity, that our probability distributions have density, although we believe this assumption is not essential.}
\begin{restatable}[Heavy-tailed ICA]{theorem}{main}\label{thm:putting_together}
Let $X=AS$ be an ICA model such that the distribution of $S$ is absolutely continuous, for all $i$ we have $\e (\abs{S_i}^{1+\gamma}) \leq M < \infty$ and normalized so that $\e \abs{S_i} = 1$, and the columns of $A$ have unit norm. 
Let $\Delta > 0$ be such that for each $i \in [n]$ if $S_i$ has finite fourth moment then its fourth cumulant satisfies 
$\abs{\cum_4(S_i)} \geq \Delta$.
Then, given $0<\eps \leq \dim^2$, $\delta > 0$, $s_M \geq \sigma_{\max}(A)$, $s_m \leq \sigma_{\min}(A)$, Algorithm~\ref{alg:orthogonalization_uniform} combined with Gaussian damping and Fourier PCA\lnote{add formal algorithm environment} outputs
$b_1, \ldots, b_n \in \R^n$ such that there are signs $\alpha_i \in \{-1,1\}$ and a permutation
$\pi:[n] \to [n]$ satisfying
\(
    \norm{A_i - \alpha_i b_{\pi(i)}} \le  \epsilon,
\)
with $\poly_\gamma(n, M, 1/s_m, s_M, 1/\Delta, R, 1/R,1/\epsilon, 1/\delta)$ time and sample complexity and 
with probability at least
$1-\delta$. Here $R$ is a parameter of the distributions of the $S_i$ as described below. The degree of the 
polynomial is $O(1/\gamma)$. 
\end{restatable}

(The assumption that $S$ has an absolutely continuous distribution is mostly for convenience in the analysis of Gaussian damping and not essential. In particular, it is not used in Algorithm~\ref{alg:orthogonalization_uniform})

Intuitively, $R$ in the theorem statement above measures how large a ball we need to restrict the 
distribution to, which has at least a constant (actually $1/\poly(n)$ suffices) probability mass 
and, moreover, each $S_i$ when restricted to the interval $[-R, R]$ has fourth cumulant at least 
$\Omega(\Delta)$. We show that all sufficiently large $R$ satisfy the above conditions and we can efficiently
compute such an $R$; see the discussion 
after Theorem~\ref{thm:ICA-orthogonal-damping} (the restatement in Sec.~\ref{sec:gaussian_damping}). 
For standard heavy-tailed distributions, such as the Pareto distribution, 
$R$ behaves nicely. For example, consider the 
Pareto distribution with shape parameter $= 2$ and scale parameter $=1$, i.e. the distribution with density 
$2/t^3$ for $t \geq 1$ and $0$ otherwise. For this distribution it's easily seen that 
$R = \Omega(\Delta^{1/2})$ suffices for the cumulant condition to be satisfied.

Theorem~\ref{thm:putting_together} requires that the $(1+\gamma)$-moment of the components $S_i$ be finite.
However, if the matrix $A$ in the ICA model is unitary (i.e. $A^TA = I$, or in other words, $A$ is a rotation matrix) then we do not need any moment assumptions:

\begin{restatable}{theorem}{gaussiandamping}\label{thm:ICA-orthogonal-damping}
Let $X=AS$ be an ICA model such that $A \in \R^{n\times n}$ is unitary (i.e., $A^TA = I$) and 
the distribution of $S$ is absolutely continuous.
Let $\Delta > 0$ be such that for each $i \in [n]$ if $S_i$ has finite fourth moment then $\abs{\cum_4(S_i)} \geq \Delta$.
Then, given $\eps, \delta > 0$,
Gaussian damping combined with Fourier PCA outputs
$b_1, \ldots, b_n \in \R^n$ such that there are signs $\alpha_i \in \{-1,1\}$ and a permutation
$\pi:[n] \to [n]$ satisfying \(
    \norm{A_i - \alpha_i b_{\pi(i)}} \le  \epsilon,
\)
in $\poly(n, R, 1/\Delta, 1/\epsilon, 1/\delta)$ time and sample complexity and 
with probability at least
$1-\delta$. Here $R$ is a parameter of the distributions of the $S_i$ as described above. 
\end{restatable}

\myparagraph{Idea of the algorithm.} Like many ICA algorithms, our algorithm has two phases: first orthogonalize the independent components (reduce to the pure rotation case), and then determine the rotation.
In our heavy-tailed setting, each of these phases requires a novel approach and analysis in the heavy-tailed setting.

A standard orthogonalization algorithm is to put $X$ in isotropic position using the covariance matrix 
$\cov(X) := \E(XX^T)$. 
This approach requires finite second moment of $X$, which, in our setting, is not necessarily finite.
Our orthogonalization algorithm (Section \ref{sec:orthogonalization_uniform}) only needs finite 
$(1+\gamma)$-absolute moment and that each $S_i$ is symmetrically distributed. 
(The symmetry condition is not needed for our ICA algorithm, as one can reduce the general case to the symmetric case, see Section \ref{sec:symmetrization}). In order to understand the first absolute moment, it is helpful to look at certain convex bodies induced by the first and second moment. The directional second moment 
$\E_X \bigl((u^TX)^2\bigr)$ is a quadratic form in $u$ and its square root is the support function of a convex body, Legendre's inertia ellipsoid, up to some scaling factor (see \cite{MilmanPajor} for example). 
Similarly, one can show that the directional absolute first moment is the support function of a convex body, the centroid body. 
When the signals $S_i$ are symmetrically distributed, the centroid body of $X$ inherits these symmetries making it absolutely symmetric (see Section \ref{sec:preliminaries} for definitions) up to an affine transformation. 
In this case, a linear transformation that puts the centroid body in isotropic position also orthogonalizes the independent components (Lemma \ref{lemma:uniform-orthogonalizer}).
In summary, the orthogonalization algorithm is the following: find a linear transformation that puts the centroid body of $X$ in isotropic position. One such matrix is given by the inverse of the square root of the covariance matrix of the uniform distribution in the centroid body. Then apply that transformation to $X$ to orthogonalize the independent components.

\details{Moreover, we do not need to use the uniform distribution in the centroid body for this purpose. 
We can use any distribution that has the same symmetries restricted to the centroid body. In particular, we could use $X$ itself restricted to the centroid body. 
There is one important caveat: the mass of $X$ inside its centroid body could be very small. 
We fix this issue by scaling up the centroid body appropriately. In summary, the orthogonalization algorithm is the following: find a linear transformation that puts $X$ restricted to a scaling of the centroid body in isotropic position. Then apply that transformation to $X$ to orthogonalize the independent components.}

We now discuss how to determine the rotation (the second phase of our algorithm). The main idea is to reduce heavy-tailed case to a case where all moments exist and to use an existing ICA algorithm (from \cite{GVX} in our case) to handle the resulting ICA instance. We use \emph{Gaussian damping} to achieve such a reduction. 
By Gaussian damping we mean to multiply the density of the orthogonalized ICA model by a spherical Gaussian density. 

We elaborate now on our contributions that make the algorithm possible.

\myparagraph{Centroid body and orthogonalization.}
The centroid body of a compact set was first defined in \cite{petty1961}. It is defined as the convex set whose support function equals the directional absolute first moment of the given compact set. We generalize the notion of centroid body to any probability measure having finite first moment (see Section \ref{sec:preliminaries} for the background on convexity and Section \ref{sec:centroidbody} for our formal definition of the centroid body for probability measures). 
\nnote{The following text can be included (I started writing it but didn't finish) For putting the centroid body 
in isotropic position we need uniformly random samples from the centroid body. There are well-known algorithms to 
do this if we have membership access to the centroid body, that is to say, there
is an efficient algorithm that given a point answers whether the point is in the body. We give such an algorithm using the ellipsoid algorithm and results from \cite{GLS}.}
In order to put the centroid body in approximate isotropic position, we estimate its covariance matrix. For this, we use uniformly random samples from the centroid body.
There are known methods to generate approximately random points from a convex body given by a membership oracle. We implement an efficient membership oracle for the centroid body of a probability measure with $1+\gamma$ moments. The implementation works by first implementing a membership oracle for the polar of the centroid body via sampling and then using it via the ellipsoid method (see \cite{GLS}) to construct a membership oracle for the centroid body.
\details{Is there an argument without polarity? Would need uniform convergence.}%
As far as we know this is the first use of the centroid body as an algorithmic tool.

An alternative approach to orthogonalization in ICA one might consider is to use the empirical covariance matrix of $X$ even when the distribution is heavy-tailed. A specific problem with this approach is that when the second moment does not exist, the diagonal entries would be very different and grow without bound. This problem gets worse when one collects more samples. This wide range of diagonal values makes the second phase of an ICA algorithm very unstable.

\myparagraph{Linear equivariance and high symmetry.} 
A fundamental property of the centroid body, for our analysis, is that the centroid body is \emph{linearly equivariant}, that is, if one applies an invertible linear transformation to a probability measure then the corresponding centroid body transforms in the same way (already observed in \cite{petty1961}). In a sense that we make precise (Lemma \ref{lem:orthogonalizer}), high symmetry and linear equivariance of an object defined from a given probability measure are sufficient conditions to construct from such object a matrix that orthogonalizes the independent components of a given ICA model. This is another way to see the connection between the centroid body and Legendre's ellipsoid of inertia for our purposes: Legendre's ellipsoid of inertia of a distribution is linearly equivariant and has the required symmetries.

\myparagraph{Gaussian damping.} 
Here we confine ourselves to the special case of ICA when the ICA matrix is unitary, that is $A^TA = I$.  
A natural idea to deal with heavy-tailed distributions is to truncate the 
distribution in far away regions and hope that the truncated distribution still gives us a way to extract 
information. In our setting, this could mean, for example, that we consider the random variable obtained from $X$ 
conditioned on the even that $X$ lies in the ball of radius $R$ centered at the origin. Instead of the ball we 
could restrict to other sets. Unfortunately, in general the resulting random variable does not come from an 
ICA model (i.e., does not have independent components in any basis). Nevertheless one may still be able to use this 
random variable for recovering $A$. We do not know how to get an algorithmic handle on it even in the case of 
unitary $A$. Intuitively, 
restricting to a set breaks the product structure of the distribution that is crucial for recovering the 
independent components. 

We give a novel technique to solve heavy-tailed ICA for unitary $A$. No moment assumptions on the components
are needed for our technique.
We call this technique \emph{Gaussian damping}. Gaussian damping can also
be thought of as restriction, but instead of being restriction to a set it is a ``restriction to a
spherical Gaussian distribution.'' Let us explain. Suppose we have a distribution on $\R^n$ with density 
$\rho_X(\cdot)$. If we restrict this distribution
to a set $A$ (which we assume to be nice: full-dimensional and without any measure theoretic issues) then the
density of the restricted distribution is $0$ outside $A$ and is proportional to $\rho_X(x)$ for $x \in A$. One
can also think of the density of the restriction as being proportional to the product of $\rho_X(x)$ and the
density of the uniform distribution on $A$. 
In the same vein, by restriction to the Gaussian distribution with density proportional to 
$e^{-\norm{x}^2/R^2}$ we simply mean the 
distribution with density proportional to $\rho_X(x) \, e^{-\norm{x}^2/R^2}$. In other words, the density of the 
restriction is obtained by multiplying the two densities. By Gaussian damping of a distribution we mean 
the distribution obtained by this operation.

Gaussian damping provides a tool to solve the ICA problem for unitary $A$ by virtue of the following properties: 
(1) \emph{The damped distribution has finite moments of all orders.} 
This is an easy consequence of the fact that Gaussian density decreases super-polynomially. More precisely, 
one dimensional moment of order $d$ given by the integral $\int_{t \in \R} t^d \rho(t) e^{-t^2/R^2} \, dt$ is 
finite for all $d \geq 0$ for any distribution. 
(2) \emph{Gaussian damping retains the product structure.} Here we use the property of spherical Gaussians 
that it's the (unique) class of spherically symmetric distributions with independent components,
i.e., the density factors: $e^{-\norm{x}^2/R^2} = e^{-x_1^2/R^2}\dotsm e^{-x_n^2/R^2}$ (we are hiding a normalizing
constant factor). Hence the damped density also factors when expressed in terms of the components of 
$s = A^{-1}x$ (again ignoring normalizing constant factors):
\begin{align*}
\rho_X(x) e^{-\norm{x}^2/R^2} 
= \rho_S(s) e^{-\norm{s}^2/R^2} = \rho_{S_1}(s_1) e^{-x_1^2/R^2} \dotsm \rho_{S_n}(S_n) e^{-x_n^2/R^2}.
\end{align*}
Thus we have converted our heavy-tailed ICA model $X=AS$ into another ICA model $X_R = A S_R$ where $X_R$ 
and $S_R$ are obtained by Gaussian damping of $X$ and $S$, resp. To this new model we can apply the existing
ICA algorithms which require at least the fourth moment to exist. This allows us to estimate matrix $A$ (up
to signs and permutations of the columns). 

It remains to explain how we get access to the damped random variable $X_R$. This is done by a simple rejection
sampling procedure. Damping does not come free and one has to pay for it in terms of higher sample 
and computational complexity, but this increase in complexity is mild in the sense that the dependence on 
various parameters of the problem is still of similar nature as for the non-heavy-tailed case. 
A new parameter $R$ is introduced here which 
parameterizes the Gaussian distribution used for damping. 
We explained the intuitive meaning of $R$ after the statement of Theorem~\ref{thm:putting_together}
and that it's a well-behaved quantity for standard distributions. 


\myparagraph{Gaussian damping as contrast function.}
Another way to view Gaussian damping is in terms of \emph{contrast functions}, a general idea that in particular
has been used fruitfully in the ICA literature. Briefly, given a function $f: \R \to \R$, for $u$ on the unit 
sphere in $\R^n$, we compute $g(u) := \E f(u^T X)$. Now the properties of the function $g(u)$ as $u$ varies 
over the unit sphere, such as its local extrema, can help us infer properties of the underlying distribution.
In particular, one can solve the ICA problem for appropriately chosen \emph{contrast function} $f$. In ICA, 
algorithms with provable guarantees use contrast functions such as moments or cumulants (e.g., \cite{dl95, FJK}). 
Many other contrast functions are also used. 
\nnote{References? This could be a good place to refer to the Belkin et al. papers.}
Gaussian damping furnishes a novel class of contrast functions that also leads to provable guarantees. 
E.g., the function given by $f(t) = t^4 e^{-t^2/R^2}$ is in this class. We do not use the contrast function view 
in this paper.

\myparagraph{Previous work related to damping.} To our knowledge the damping technique, and more generally
the idea of reweighting the data, is new in the context of ICA. But the general idea of reweighting is not new 
in other somewhat related contexts as we now discuss.
In robust statistics (see, e.g.,~\cite{Huber}), reweighting idea is used for outlier removal by giving less 
weight to far away data points. Apart from this high-level similarity we are not aware of any closer 
connections to our setting; in particular, the weights used are generally different. 

Another related work is \cite{BrubakerVempala}, on isotropic PCA, affine invariant clustering, and learning mixtures of Gaussians. This work uses Gaussian reweighting. However, again we are unaware of any more specific connection to our problem. 

Finally, in \cite{Yeredor, GVX} a different reweighting, using a ``Fourier weight'' $e^{i u^T x}$ (here $u \in \R^n$
is a fixed vector and $x \in \R^n$ is a data point) is used in the computation of the covariance matrix. This 
covariance matrix is useful for solving the ICA problem. But as discussed before, results here do not seem to be
amenable to our heavy-tailed setting.

\nnote{Discussion of how to get a handle on the centroid body: dual body and ellipsoid.
Discussion of floating body versus centroid.
Discussion of why dual characterization of cvar is not used.}

\myparagraph{Organization.} After preliminaries in the next section, in Sec.~\ref{sec:centroidbody} we 
define the centroid body and prove some useful properties. In Sec.~\ref{subsec:membership-centroid-body}
we show how to construct the membership oracle for the centroid body of a distribution produced by 
a symmetric ICA model. In Sec.~\ref{sec:orthogonalization_uniform} we use this membership oracle to compute
the covariance matrix of the uniform distribution on the centroid body and using this matrix we orthogonalize
the independent components. Finally, in Sec.~\ref{sec:symmetrization} we show why working with symmetric ICA
model is without loss of generality. 
\ifshort{Because of the space constraints, we omit all the proofs in the
following sections. Gaussian damping, which was discussed in some detail in the introduction, has to be omitted
from this extended abstract; also omitted is the composition of orthogonalization with Gaussian damping to 
prove our main theorem. Our choice of sections to include in this extended abstract does not reflect the 
relative importance of the ideas in these sections and was rather dictated by space constraints.}\fi

\section{Preliminaries}\label{sec:preliminaries}
This section contains some basic notation, definitions, and results from previous work.

We will denote random variables by capital letters, e.g. $X, S$, and the values they might take by
corresponding lower case letters, e.g., $x, s$. For a random variable $X$, let $\measure_X$ denote the probability measure induced by $X$. 
We take all vectors to be column vectors, and for a vector $x$, by $\norm{x}$ we mean $\norm{x}_2$.
For two vectors $x, y \in \RR^{\dim}$ we let $\inner{x}{y} = x^T y$ denote their inner product. 
For $n \in \mathbb{N}$, let $[n]$ denote the set of integers $1, \dots, n$. The symbol
 $B_{p}^{\dim}$ stands for
the $\dim$-dimensional unit $\ell_p$-ball and $S^{n-1}$ for the $\ell_2$-unit sphere in $\R^{\dim}$.

\An\ $\dim$-dimensional \emph{convex body} is a compact convex subset of $\RR^\dim$ with non-empty interior.
We say a convex body $K \subseteq \RR^\dim$ is \emph{absolutely symmetric} if $(x_1, \dots, x_{\dim}) \in K \Leftrightarrow (\pm x_1, \dots, \pm x_{n}) \in K$.
Similarly, we say random variable $X$ (and its distribution) is absolutely symmetric if, for any choice of signs $\alpha_i \in \{-1, 1\}$, $(x_1, \dots, x_n)$ has the same distribution as $(\alpha_1 x_1, \dots, \alpha_n x_n)$.
We say that \an\ $\dim$-dimensional random vector $X$ is \emph{symmetric} if 
$X$ has the same distribution as $-X$.
Note that if $X$ is symmetric with independent components (mutually independent coordinates) then its components are also symmetric.

Let $K \subseteq \RR^d$ be a non-empty set. 
The set
\(
K^\polar := \{x \in \RR^\dim \suchthat \langle x, y \rangle \leq 1 \; \forall y \in K\}
\)
is called the \emph{polar} of $K$. 
The \emph{support function} of $K$ is $h_{K}: \RR^\dim \to \RR$ and defined by
\( h_{K}(\theta) = \sup_{x \in K}\, \inner{x}{\theta}. \)
The \emph{radial function} of $K$ is $r_{K}: \RR^{\dim}\setminus \{0\} \to \RR$ and defined by
\(
r_{K}(\theta) = \sup \, \{\alpha \in \RR \suchthat \alpha \theta \in K \}.
\)
For a random variable $X \in \R$, let $m_i(X) = \E X^i$ be its $i$th moment.
For each positive integer $i$ the cumulant of order $i$ of r.v. $X$, denoted $\cum_i(X)$, is a polynomial involving the moments of $X$. 
The \emph{fourth cumulant} of $X$ is equal to 
$m_4(X)-4m_3(X)m_1(X)-3m_2(X)^2+12m_2(X)m_1(X)^2 - 6m_1(X)^4$. When $X$ is symmetric, this simplifies 
to $m_4(X) - 3m_2(X)^2$. Cumulants can be thought of as a higher degree generalization of variance, and they have 
some nice properties that make them useful in ICA, e.g, if $X$ and $Y$ are independent random 
variables then 
$\cum_4(X+Y) = \cum_4(X) + \cum_4(Y)$.  Another useful property is that if $X$ is a Gaussian random variable then
$\cum_4(X) = 0$. For this reason, the absolute value of the fourth cumulant is often used to quantify the distance
of the distribution of a random variable from the set of Gaussian distributions. 



\begin{definition}[(Symmetric) ICA model]\label{def:ica}
Let $A \in \RR^{\dim \times \dim}$ be an invertible matrix and let $S \in \RR^{\dim}$ be a random vector 
whose coordinates $S_i$ are mutually independent.
We then say that the random vector $X = AS$ is given by an \emph{ICA model}. If, in addition, $S$ is symmetric
(equivalently, using the independence of the components $S_i$, each component $S_i$ is symmetric) then we say 
that the random vector $X = AS$ is given by a \emph{symmetric ICA model}.
\end{definition}

For $X$ given by a symmetric ICA model
we say the matrix $B$ is an \emph{orthogonalizer} of $X$ if the columns of $BA$ are orthogonal.

We say that a matrix $A \in \R^{n\times n}$ is \emph{unitary} if $A^TA = I$, or in other words $A$ is a 
rotation matrix. (Normally for matrices
with real-valued entries such matrices are called orthogonal matrices and the word unitary is reserved for  
their complex counterparts, however the word orthogonal matrix can lead to confusion in the present paper.)
For matrix $C \in \R^{n \times n}$, denote by $\norm{C}_2$ the spectral norm and
by $\norm{C}_F$ the Frobenius norm.
\iflong
We will need the following inequality about the stability of matrix inversion (see for example 
\cite[Chapter III, Theorem 2.5]{stewart1990matrix}).

\begin{lemma}\label{lem:inversion}
Let $\norm{\cdot}$ be a matrix norm such that $\norm{AB} \leq \norm{A} \norm{B}$. Let matrices $C, E \in \R^{n\times n}$ be such that $\norm{C^{-1} E}_2 \leq 1$, and
let $\tilde{C} = C + E$. Then
\begin{equation}
\frac{\norm{\tilde{C}^{-1} - C^{-1}}} {\norm{C^{-1}}} \leq \frac{\norm{C^{-1} E}}{1 - \norm{C^{-1} E}}.
\end{equation}
\end{lemma}

This implies that if $\norm{E}_2 = \norm{\tilde{C} - C}_2 \leq 1/(2 \norm{C^{-1}}_2)$, then
\begin{equation}\label{eq:inverse-stability}
\norm{\tilde{C}^{-1} - C^{-1}}_2 \leq 2 \norm{C^{-1}}_{2}^{2} \norm{E}_{2}.
\end{equation}
\fi

\subsection{Results from convex optimization}


We need the following result: Given a membership oracle for a convex body $K$ one can implement efficiently a membership oracle for $K^\polar$, the polar of $K$. 
This follows from applications of the ellipsoid method from \cite{GLS}.
Specifically, we use the following facts: (1) a validity oracle for $K$ can be constructed from a membership oracle for $K$ \cite[Theorem 4.3.2]{GLS}; (2) a membership oracle for $K^\polar$ can be constructed from a validity oracle for $K$ \cite[Theorem 4.4.1]{GLS}.

The definitions and theorems in this section all come (occasionally with slight rephrasing) from \cite{GLS} except for the notion of $(\epsilon, \delta)$-weak oracle. 
[The definitions below use rational numbers instead of real numbers. 
This is done in \cite{GLS} as they work out in detail the important low level issues of how the numbers 
in the algorithm are represented as general real numbers cannot be directly handled by computers. 
These low-level details can also be worked out for the arguments in this paper, but as is customary, we will 
not describe these and use real numbers for the sake of exposition.]
In this section $K \subseteq \R^n$ is a convex body. For $y \in \R^n$, the distance of $y$ to $K$ is given by
$d(y, K) := \min_{z \in K}\,\norm{y-z}$. 
Define $S(K,\eps) := \{y \in \RR^n \suchthat d(y,K) \leq \eps \}$
and
$S(K,-\eps) := \{y \in \RR^n \suchthat S(y,\eps) \subseteq K \}$. Let $\QQ$ denote the set of rational numbers.

\begin{definition}[\cite{GLS}]\label{def:eps-weak-oracle}
The \emph{$\epsilon$-weak membership problem} for $K$ is the following:
Given a point $y \in \QQ^n$ and a rational number $\eps > 0$, 
either (i) assert that $y \in S(K,\eps)$, or (ii) assert that $y \not \in S(K,-\eps)$.
An \emph{$\epsilon$-weak membership oracle} for $K$ is an oracle that solves the weak membership problem for $K$.
For $\delta \in [0,1]$, an \emph{$(\eps, \delta)$-weak membership oracle} for $K$ acts as 
follows: Given a point $y \in \QQ^n$, with probability at least $1-\delta$ it solves the  
$\eps$-weak membership problem for $y, K$, and otherwise its output can be arbitrary.  
\end{definition}



\begin{definition}[\cite{GLS}]\label{def:wval}
The \emph{$\eps$-weak validity problem} for $K$ is the following: 
Given a vector $c \in \QQ^n$, a rational number $\gamma$, and a rational number $\eps > 0$, either
(i) assert that $c^T x \leq \gamma + \eps$ for all $x \in S(K, -\eps)$, or
(ii) assert that $c^T x \geq \gamma - \eps$ for some $x \in S(K, \eps)$.
The notion of $\epsilon$-weak validity oracle and $(\epsilon, \delta)$-weak validity oracle can be defined
similarly to Def.~\ref{def:eps-weak-oracle}. 
\end{definition}



\begin{definition}[{\cite[Section 2.1]{GLS}}]
We say that an oracle algorithm is an \emph{oracle-polynomial time algorithm} for a certain problem defined on a class of convex sets if the running time of the algorithm is bounded by a polynomial in the encoding length of $K$ and in the encoding length of the possibly existing further input, for every convex set $K$ in the given class.
\end{definition}
The \emph{encoding length} of a convex set will be specified below, depending on how the convex set is presented.
\nnote{check that this has been done}

\begin{theorem}[Theorem 4.3.2 in \cite{GLS}]\label{thm:wmem-to-wviol}
Let $R > r > 0$ and $a_0 \in \R^n$. 
There exists an oracle-polynomial time algorithm that solves the weak validity problem for every convex body $K \subseteq \RR^n$ contained in the ball of radius $R$ and containing a ball of radius $r$ centered at $a_0$ given by a weak membership oracle. The encoding length of $K$ is $n$ plus the length of the binary encoding of $R, r$, and $a_0$.
\end{theorem}
We remark that the Theorem~4.3.2 as stated in \cite{GLS} is stronger than the above statement 
in that it constructs a weak violation oracle (not defined here) 
which gives a weak validity oracle which suffices for us.
The algorithm given by Theorem 4.3.2 makes a polynomial (in the encoding length of $K$) number of queries to
the weak membership oracle. 

\begin{lemma}[Lemma 4.4.1 in \cite{GLS}]\label{lem:wviol-polar-wmem}
There exists an oracle-polynomial time algorithm that solves the weak membership problem for $K^\polar$, where $K$ is a convex body contained in the ball of radius $R$ and containing a ball of radius $r$ centered at $0$ given by a weak validity oracle.
The encoding length of $K$ is $n$ plus the length of the binary encoding of $R$ and $r$.
\end{lemma}
Our algorithms and proofs will need more quantitative details from the proofs of the 
above theorem and lemma. These will be mentioned when we need them.

\iflong
\subsection{Results from algorithmic convexity}
We state here a special case of a standard result in algorithmic convexity: There is an efficient algorithm to estimate the covariance matrix of the uniform distribution in a centrally symmetric convex body given by a weak membership oracle.
\nnote{Does the body need to be centrally symmetric? Also need to say that by the covariance matrix of a convex body we mean the 
covariance matrix of the uniform distribution on the body.}
The result follows from the random walk-based algorithms to generate approximately uniformly random points from a convex body \cite{DBLP:journals/jacm/DyerFK91, MR1608200, DBLP:journals/jcss/LovaszV06}. 
Most papers use access to a membership oracle for the given convex body. In this paper we only have access to an $\eps$-weak membership oracle. 
As discussed in \cite[Section 6, Remark 2]{DBLP:journals/jacm/DyerFK91}, essentially the same algorithm implements efficient sampling when given an $\eps$-weak membership oracle.
The problem of estimating the covariance matrix of a convex body was introduced in \cite[Section 5.2]{MR1608200}. 
That paper analyzes the estimation from random points. 
There has been a sequence of papers studying the sample complexity of this problem \cite{MR1665576, MR1694526, MR2190337, MR2276533, MR2276637, MR2956207, adamczak, MR3127875}.
\nnote{I'm not sure we need all these references. Maybe just refer to one of the latest ones and
say see references therein.}

\begin{theorem}\label{thm:covariance_estimation}
Let $K \subseteq \RR^\dim$ be a centrally symmetric convex body given by a weak membership oracle so that $r B_2^\dim \subseteq K \subseteq R B_2^\dim$. Let $\Sigma = \cov(K)$. Then there exists a randomized algorithm that, when given access to the weak membership of $K$ and inputs $r, R, \delta, \eps_c >0$, it outputs a matrix $\tilde \Sigma$ such that with probability at least $1-\delta$ over the randomness of the algorithm\details{and for any response by the weak membership oracle to queries in the ambiguous region}, we have 
\begin{equation}\label{equ:covariance}
(\forall u \in \RR^\dim) \qquad (1-\eps_c) u^T \Sigma u \leq u^T \tilde \Sigma u \leq (1+\eps_c) u^T \Sigma u.
\end{equation}
The running time of the algorithm is $\poly(\dim, \log(R/r), 1/\eps_c, \log(1/\delta))$.
\end{theorem}
Note that \eqref{equ:covariance} implies $\norm{\tilde \Sigma - \Sigma}_2 \leq \eps_c \norm{\Sigma}_2$. \details{\eqref{equ:covariance} $\implies$ $\abs{u^T \tilde \Sigma u - u^T \Sigma u} \leq \eps_c u^T \Sigma u \implies \abs{u^T (\tilde \Sigma -  \Sigma) u} \leq \eps_c u^T u \norm{\Sigma}$. The claim follows from the Rayleigh quotient characterization of eigenvalues and the spectral norm.}
The guarantee in \eqref{equ:covariance} has the advantage of being invariant under linear transformations in the following sense: 
if one applies an invertible linear transformation $C$ to the underlying convex body, the covariance matrix and its estimate become $C \Sigma C^T$ and $C \tilde \Sigma C^T$, respectively. These matrices satisfy 
\begin{equation*}
(\forall u \in \RR^\dim) \qquad (1-\eps_c) u^T C \Sigma C^T u \leq u^T C \tilde \Sigma C^T u \leq (1+\eps_c) u^T C \Sigma C^T u.
\end{equation*}
This fact will be used later. \nnote{Where is this used? As far as I can see, only the sentence right after
the theorem statement gets used.}
\fi

\section{The centroid body}\label{sec:centroidbody}

The main tool in our orthogonalization algorithm is the centroid body of a distribution, which we use as a first moment analogue of the covariance matrix.
In convex geometry, the centroid body is a standard (see, e.g., \cite{petty1961,MilmanPajor,gardner1995geometric})  convex body associated to (the uniform distribution on) a given convex body.
Here we use a generalization of the definition from the case of the uniform distribution on a convex body to more general probability measures. \nnote{why switch the language to probability measures?}
Let $X \in \RR^\dim$ be a random vector with finite first moment, that is, for all $u \in \RR^\dim$ we have $\e (\abs{\inner{u}{X}}) < \infty$. Following \cite{petty1961}, consider the function
\(
h(u) = \e (\abs{\inner{u}{X}})
\).
Then it is easy to see that $h(0) = 0$, $h$ is positively homogeneous, and $h$ is subadditive. Therefore, it is the support function of a compact convex set \cite[Section 3]{petty1961}, \cite[Theorem 1.7.1]{MR1216521}, \cite[Section 0.6]{gardner1995geometric}. \nnote{I think the fact that the centroid body of a distribution is convex is a key fact and did not appear in a previous paper even though the proof is basically the same as that for convex bodies (correct? I actually haven't checked it); so it should be emphasized by making it a lemma/prop. and the proof included although it can go to the appendix. Also perhaps explain why it works for the first moment and what happens for higher and lower ones.}
This justifies the following definition:
\begin{definition}[Centroid body
]\label{def:centroid-body2}
Let $X \in \RR^\dim$ be a random vector with finite first moment, that is, for all $u \in \RR^\dim$ we have $\e (\abs{\inner{u}{X}}) < \infty$. 
The \emph{centroid body} of $X$ is the compact convex set, denoted $\Gamma X$, whose support function is $h_{\Gamma X}(u) = \e (\abs{\inner{u}{X}})$.
For a probability measure $\measure$, we define $\Gamma \measure$, the centroid body of $\measure$, as the centroid body of any random vector distributed according to $\measure$.
\end{definition}

The following lemma says that the centroid body is equivariant under linear transformations. 
It is a slight generalization of statements in \cite{petty1961} and \cite[Theorem 9.1.3]{gardner1995geometric}.

\begin{lemma}\label{lem:equivariance}
Let $X$ be a random vector on $\RR^\dim$.
Let $A: \RR^\dim \to \RR^\dim$ be an invertible linear transformation.
Then $\Gamma (AX) = A (\Gamma X)$.
\end{lemma}
\begin{proof}
$
x \in \Gamma(AX) \Leftrightarrow
\forall u \inner{x}{u} \leq \e (\abs{\inner{u}{AX}}) \Leftrightarrow
\forall u \inner{x}{u} \leq \e (\abs{\inner{A^T u}{X}}) \Leftrightarrow
\forall v \inner{x}{A^{-T}v} \leq \e (\abs{\inner{v}{X}}) \Leftrightarrow
\forall v \inner{A^{-1}x}{v} \leq \e (\abs{\inner{v}{X}}) \Leftrightarrow
A^{-1}x \in \Gamma(X) \Leftrightarrow
x \in A\Gamma(X)
$.
\end{proof}
\details{it seems to me (Luis) that linear equivariance is true even in the non-invertible case, but not through that proof} 
\begin{lemma}[{\cite[Section 0.8]{gardner1995geometric}, \cite[Remark 1.7.7]{MR1216521}}]\label{lem:polar-radial}
Let $K \subseteq \RR^\dim$ be a convex body with support function $h_K$ and such that the origin is in the interior of $K$.
Then $K^\polar$ is a convex body and has radial function $r_{K^\polar}(u) = 1/h_{K}(u)$ for $u \in S^{n-1}$.
\end{lemma}
Lemma~\ref{lem:polar-radial} implies that testing membership in $K^{\polar}$ 
is a one-dimensional problem if we have access to the support function $h_K(\cdot)$:
We can decide if a point $z$ is in  $K^{\polar}$ 
by testing if $\norm{z} \leq 1/h_K({z}/{\norm{z}})$ instead of needing to check $\inner{z}{u} \leq 1$ for all $u \in K$.
In our application $K = \Gamma X$. We can estimate $h_{\Gamma X}({z}/{\norm{z}})$ by taking the empirical average of  
$\abs{\inner{x^{(i)}}{{z}/{\norm{z}}}}$ where the $x^{(i)}$ are samples of $X$. This leads to an approximate oracle for $(\Gamma X)^\polar$
which will suffice for our application. The details are in Sec.~\ref{subsec:membership-centroid-body}. 


\iflong
\section{Mean estimation using \texorpdfstring{$1 + \gamma$}{1+g} moments}
We will need to estimate the support function of the centroid body in various directions. To this end
we need to estimate the first absolute moment of the projection to a direction. Our assumption that 
each component $S_i$ has finite $(1+\gamma)$-moment will allow us to do this with a reasonable small 
probability of error. This is done via the following Chebyshev-type inequality.

\nnote{I think this section, except for the statement of the lemma, should go to the appendix. Not sure yet where the lemma statement
should go.}

Let $X$ be a real-valued symmetric random variable such that $\E \abs{X}^{1+\gamma} \leq M$ for some $M > 1$ and 
$0 < \gamma < 1$. Then we will prove that the empirical average of the expectation of $X$ converges to the expectation of $X$.

Let $\tilde{\E}_N[\abs{X}]$ be the empirical average obtained from $N$ independent samples 
$X^{(1)}, \ldots, X^{(N)}$,
i.e., $(\abs{X^{(1)}}+\dotsb+\abs{X^{(N)}})/N$.

\begin{lemma}\label{lem:1-plus-eps-chebyshev}
Let $\epsilon \in (0,1)$. With the notation above, for 
$N \geq \left(\frac{8M}{\epsilon}\right)^{\frac{1}{2}+\frac{1}{\gamma}}$,
we have
\begin{align*}
\Pr[\abs{\tilde{\E}_N[\abs{X}]-\E[\abs{X}]} > \epsilon] \leq \frac{8M}{\epsilon^2 N^{\gamma/3}}.
\end{align*}
\end{lemma}

\begin{proof}
Let $T > 1$ be a threshold whose precise value we will choose later. We have
\[ \Pr[\abs{X} \geq T] \leq \frac{\E[\abs{X}^{1+\gamma}]}{T^{1+\gamma}} = \frac{M}{T^{1+\gamma}}. \]
By the union bound,
\begin{align} \label{eqn:union}
\Pr[\exists i \in [N] \mbox{ such that } \abs{X^{(i)}} > T] \leq \frac{NM}{T^{1+\gamma}}.
\end{align}
Define a new random variable $X_T$ by 
\begin{align*}
X_T = \begin{cases} X & \text{if $\abs{X} \leq T$,} \\
                    0 & \text{otherwise.}
\end{cases}
\end{align*}
Using the symmetry of $X_T$ we have
\begin{align}\label{eqn:X_T_upper}
\Var[\abs{X_T}] \leq \E[X_T^2] \leq T \, \E[\abs{X_T}] \leq T\, \E[\abs{X}] \leq T M^{1/(1+\gamma)}. 
\end{align}
By the Chebyshev inequality and \eqref{eqn:X_T_upper} we get 
\begin{align} \label{eqn:chebyshev}
\Pr[\abs{\tilde{\E}_N[\abs{X_T}] - \E[\abs{X_T}]} > \epsilon'] 
\leq \frac{\Var[\abs{X_T}]}{N \epsilon'^2} 
\leq \frac{T M^{1/(1+\gamma)}}{N \epsilon'^2}.
\end{align}
Putting \eqref{eqn:union} and \eqref{eqn:chebyshev} together for $0 < \epsilon' < 1/2$ we get
\begin{align} \label{eqn:union+chebyshev}
  \Pr[\abs{\tilde{E}_N[\abs{X}]-\E[\abs{X_T}]} > \epsilon'] \leq \frac{NM}{T^{1+\gamma}} 
+ \frac{T M^{1/(1+\gamma)}}{N \epsilon'^2}.
\end{align}
Choosing 
\begin{align}\label{eqn:N_T}
N := \frac{T^{1+\gamma/2}}{\epsilon' M^{\gamma/(2(1+\gamma))}}, 
\end{align}
the RHS of the previous equation becomes $\frac{2 M^{1-\gamma/(2(1+\gamma))}}{\epsilon' T^{\gamma/2}}$. The choice of $N$ is made to minimize the RHS; we ignore integrality issues.
Pick $T_0>0$ so that $\abs{\E[\abs{X}]-\E[\abs{X_{T_0}}]} < \epsilon'$. 
To estimate $T_0$, note that
\begin{align*}
M = \E[\abs{X}^{1+\gamma}] \geq T_0^\gamma \, \E[\abs{\abs{X}-\abs{X_{T_0}}}],
\end{align*}
Hence
\begin{align} \label{eqn:mean_X_X_T}
\abs{\E[\abs{X}]-\E[\abs{X_{T_0}}]} \leq \E[\abs{\abs{X} - \abs{X_{T_0}}}] \leq M/T_0^\gamma.
\end{align}
We want $M/T_0^\gamma \leq \epsilon'$ which is equivalent to $T_0 \geq (\frac{M}{\epsilon'})^{1/\gamma}$. 
We set $T_0 := (\frac{M}{\epsilon'})^{1/\gamma}$. 
Then, for $T \geq T_0$ putting together \eqref{eqn:union+chebyshev} and \eqref{eqn:mean_X_X_T} gives
\begin{align*}
\Pr[\abs{\tilde{\E}_N[\abs{X}]-\E[\abs{X}]} > 2\epsilon'] \leq \frac{2 M^{1-\gamma/(2(1+\gamma))}}{\epsilon' T^{\gamma/2}}.
\end{align*}
Setting $\epsilon = 2\epsilon'$ (so that $\epsilon \in (0,1)$) and expressing the RHS of the last 
equation in terms of $N$ via 
\eqref{eqn:N_T}
(and eliminating $T$), 
and using our assumptions $\gamma, \epsilon \in (0,1)$, $M>1$ to get a simpler upper bound, we get
\begin{align*}
\Pr[\abs{\tilde{\E}_N[\abs{X}]-\E[\abs{X}]} > \epsilon] &\leq
\frac{2^{2+\frac{\gamma}{(2+\gamma)}} M^{1-\frac{\gamma}{2(1+\gamma)} + \frac{\gamma^2}{2(2+\gamma)(1+\gamma)}}}
{\epsilon^{1+\frac{\gamma}{2+\gamma}} N^{\frac{\gamma}{2+\gamma}}} \leq \frac{8M}{\epsilon^2 N^{\gamma/3}}.
\end{align*}
Condition $T \geq T_0$, when expressed in terms of $N$ via \eqref{eqn:N_T}, becomes
\begin{align*}
N \geq \frac{2^{\frac{3}{2}+\frac{1}{\gamma}}}{\epsilon^{\frac{1}{2}+\frac{1}{\gamma}}} 
M^{\frac{1}{2} + \frac{1}{\gamma} - \frac{\gamma}{2(1+\gamma)}} \geq \left(\frac{8M}{\epsilon}\right)^{\frac{1}{2}+\frac{1}{\gamma}}.
\end{align*}
\end{proof}


\fi

\section{Membership oracle for the centroid body} \label{subsec:membership-centroid-body}
In this section we provide an efficient weak membership oracle (Subroutine~\ref{sub:cvar-oracle}) for the
centroid body $\Gamma X$ of the r.v. $X$.
This is done by first providing a weak membership oracle (Subroutine~\ref{sub:polar-cvar-oracle}) for the polar body
$(\Gamma X)^\polar$.
We begin with a lemma that shows that under certain general conditions the centroid body is 
``well-rounded.'' This property will prove useful in the membership tests.

\begin{lemma}\label{lem:centroid-scaling}
Let $S = (S_1, \dots, S_n) \in \RR^n$ be an absolutely symmetrically distributed random vector such that
$\e (\abs{S_i}) = 1$ for all $i$.
Then $B_{1}^{\dim} \subseteq \Gamma S \subseteq [-1,1]^\dim$.
Moreover, $\dim^{-1/2} B_2^\dim \subseteq (\Gamma S)^\polar \subseteq \sqrt{\dim}B_2^\dim$.
\end{lemma}

\begin{proof}
The support function of $\Gamma S$ is $h_{\Gamma S}(\theta) = \e \abs{\inner{S}{\theta}}$ (Def.~\ref{def:centroid-body2}).
Then, for each canonical vector $e_i$, $h_{\Gamma S}(e_i) = h_{\Gamma S}(-e_i) =  \E\abs{S_i} = 1$.
Thus, $\Gamma S$ is contained in $[-1,1]^\dim$. 
Moreover, since $[-1,1]^\dim \subseteq \sqrt{\dim} B_2^\dim$, we get $\Gamma S \subseteq \sqrt{\dim} B_2^\dim$.

We claim now that each canonical vector $e_i$ is contained in $\Gamma S$.
To see why, first note that since the support function is $1$ along each canonical direction, $\Gamma S$ will touch the facets of the unit hypercube.
Say, there is a point $(1, x_2, x_3, \dots, x_n)$ that touches facet associated to canonical vector $e_1$.
But the symmetry of the $S_i$s implies that $\Gamma S$ is absolutely symmetric, so that $(1,\pm x_2, \pm x_3, ..., \pm x_n)$ is also in the centroid body.
Convexity implies that $(1,0,\dots,0) = e_1$ is in the centroid body.
The same argument applied to all $\pm$ canonical vectors implies that they are all contained in the centroid body, and this with convexity implies that the centroid body contains $B_1^\dim$.
In particular, it contains $n^{-1/2} B_2^\dim$.
\end{proof}


\subsection{Membership oracle for the polar of the centroid body}
\nnote{This section doesn't seem to need symmetry. Or does it?}
As mentioned before, our membership oracle for $(\Gamma X)^\polar$ (Subroutine~\ref{sub:polar-cvar-oracle}) 
is based on the fact that $1/h_{\Gamma X}$ is the radial function of $(\Gamma X)^\polar$, and that $h_{\Gamma X}$ is the directional absolute first moment of $X$, which can be efficiently estimated by sampling. 
\begin{subroutine}[ht]
\caption{Weak Membership Oracle for $(\Gamma X)^{\polar}$}
\label{sub:polar-cvar-oracle}
\begin{algorithmic}[1]
\Require
Query point $y \in \QQ^\dim$, samples from symmetric ICA model $X = AS$,
bounds $s_M \geq \sigma_{\max}(A)$, $s_m \leq \sigma_{\min}(A)$,
closeness parameter $\eps$, failure probability $\delta$.
\Ensure Weak membership decision for $y \in (\Gamma X)^{\polar}$.
\State Generate iid samples $x^{(1)}, x^{(2)}, \dots, x^{(N)}$ of $X$ for $N = \poly_{\gamma}(n, M, 1/s_m, s_M, 1/\eps, 1/\delta)$.
\State Compute 
\[ 
\tilde{h} = \frac{1}{N} \sum_{i=1}^N \abs{\inner{x^{(i)}}{\frac{y}{\norm{y}}}}.
\]
\State If $\norm{y} \leq 1/\tilde{h}$, report $y$ as feasible. Otherwise, report $y$ as infeasible.
\end{algorithmic}
\end{subroutine}

\begin{lemma}[Correctness of Subroutine~\ref{sub:polar-cvar-oracle}]
Let $\gamma > 0$ be a constant and $X=AS$ be given by a symmetric ICA model such that for all $i$ we have $\e (\abs{S_i}^{1+\gamma}) \leq M < \infty$ and normalized so that $\e \abs{S_i} = 1$.
Let $\eps, \delta > 0$.
Given $s_M \geq \sigma_{\max}(A)$ , $s_m \leq \sigma_{\min}(A)$,
Subroutine~\ref{sub:polar-cvar-oracle} is an $(\eps, \delta)$-weak membership oracle for $(\Gamma X)^{\polar}$ with
using time and sample complexity
$\poly_{\gamma}(n, M, 1/s_m, s_M, 1/\eps, 1/\delta)$. The degree of the polynomial is $O(1/\gamma)$.
\nnote{Revisit the dependence on gamma. Propagate this dependence to other things.}
\end{lemma}

\begin{proof}
Recall from Def.~\ref{def:eps-weak-oracle} that we need to show that, with probability at least $1-\delta$, Subroutine~\ref{sub:polar-cvar-oracle} outputs TRUE when $y \in S((\Gamma X)^{\polar}, \eps)$ and FALSE when $y \not \in S((\Gamma X)^{\polar}, -\eps)$; otherwise, the output can be either TRUE or FALSE arbitrarily.

Fix a point $y \in \QQ^\dim$ and let $\theta := y/\norm{y}$ denote the direction of $y$.
The algorithm estimates the radial function of $(\Gamma X)^{\polar}$ along $\theta$, which is $1/h_{\Gamma X}(\theta)$ (see Lemma~\ref{lem:polar-radial}).
In the following computation, we simplify the notation by using $h = h_{\Gamma X}(\theta)$.
It is enough to show that with probability at least $1-\delta$ the algorithm's estimate, $1/\tilde h$, of the radial function is within $\eps$ of the true value, $1/h$.
\details{Suppose $y$ is in the inner parallel body. Then $1/h = r_{\centroid X}(\theta) \geq \norm{y} + \eps$. Suppose $y$ is not in the outer parallel body. Then $r_{\centroid X}(\theta) +\eps < \norm{y}$.}

For $X^{(1)}, \dots, X^{(N)}$, i.i.d. copies of $X$, the empirical estimator for $h$ is
$\tilde{h} := \frac{1}{N} \sum_{i=1}^{N} \abs{\inner{X^{(i)}}{\theta}}$.

We want to apply Lemma~\ref{lem:1-plus-eps-chebyshev} to $\inner{X}{\theta}$. For this we need a bound on its 
$(1+\gamma)$-moment. 
The following simple bound is sufficient for our purposes: Let $u = A^T \theta$. Then $\abs{u_i} \leq \sigma_{\max}(A)$ for all $i$. Then
\begin{equation}\label{eq:momentX}
\begin{aligned}
\e \lrabs{\inner{X}{\theta}}^{1+\gamma} 
&= \e \lrabs{\theta^T A S}^{1+\gamma} 
= \e \lrabs{\inner{S}{u}}^{1+\gamma} 
= \e \abs{\sum_{i=1}^\dim S_i u_i}^{1+\gamma} 
\leq \sum_{i=1}^\dim \e \lrabs{S_i u_i}^{1+\gamma} \\
&= \sum_{i=1}^\dim \e \lrabs{S_i}^{1+\gamma} \lrabs{u_i}^{1+\gamma} 
\leq M \sum_{i=1}^\dim \lrabs{u_i}^{1+\gamma} 
\leq M \dim \sigma_{\max}(A)^{1+\gamma} \leq M \dim s_M^{1+\gamma}.
\end{aligned}
\end{equation}
Lemma~\ref{lem:1-plus-eps-chebyshev} implies that, for $\eps_1>0$ to be fixed later, and for
\begin{align} \label{eqn:N-lower-bound}
N > \left(\frac{8 M \dim s_M^{1+\gamma}}{\eps_1^2 \delta}\right)^{3/\gamma}, 
\end{align}
we have $P(|\tilde{h} - h| > \eps_1) \leq \delta$.

\nnote{I think some more details should be given above.}

Lemmas~\ref{lem:centroid-scaling} and \ref{lem:equivariance} give that $r B_2^\dim \subseteq \Gamma X$ for $r := s_m/\sqrt{\dim} \leq \sigma_{\min}(A)/\sqrt{\dim}$.
It follows that $h \geq r$. If $|\tilde{h} - h| \leq \eps_1$ and $\eps_1 \leq r/2$, then we have
\[
\left\lvert\frac{1}{h} - \frac{1}{\tilde h} \right\rvert= \frac{\abs{h - \tilde h}}{h \tilde h} \leq \frac{\eps_1}{r (r-\eps_1)} \leq \frac{2 \eps_1}{r^2},
\]
which in turn gives, when $\eps_1 = \min \{r^2 \eps/2, r/2\}$,
\[
P\left(\left|\frac{1}{\tilde{h}} - \frac{1}{h}\right| \leq \eps\right)
\geq P\left(\left|\frac{1}{\tilde{h}} - \frac{1}{h}\right| \leq \frac{2\eps_1}{r^2}\right)
\geq P(|\tilde{h} - h| \leq \eps_1) \geq 1 - \delta.
\]
\nnote{I don't follow ``details'' below and also some of the main text which I moved to ``details'' and I don't see the need for it. If you agree then ``details'' below can be removed.}
\details{To conclude we can assume that $\eps$ is at most half of the circumradius of $(\Gamma X)^\polar$. The circumradius is at most $1/r$. 
Thus, $r^2 \eps/2 \leq r/2$ and it is enough to take $\eps_1 = r^2 \eps /2 \leq \sigma_{\min}(A)^2\eps/(4\dim)$}
\details{Thus, it is enough to take $\eps_1 \leq \min\{ s_m^2\eps/(4\dim) , s_m/(2\sqrt{\dim}) \} \leq \min\{ \sigma_{\min}(A)^2\eps/(4\dim) , \sigma_{\min}(A)/(2\sqrt{\dim}) \}$} Plugging in the value of 
$\epsilon_1$ and, in turn of $r$, into \eqref{eqn:N-lower-bound} gives that it suffices to take 
\[N > \poly_\gamma(n, M, 1/s_m, s_M, 1/\eps, 1/\delta). 
\]
%
\end{proof}

\subsection{Membership oracle for the centroid body}  
We now describe how the weak membership oracle for the centroid body $\Gamma X$ is constructed using 
the weak membership oracle for $(\Gamma X)^\polar$, provided by Subroutine~\ref{sub:polar-cvar-oracle}. 

We will use the following notation: For a convex body $K \in \R^n$, $\eps, \delta >0$, $R \geq r > 0$ such that 
$r B_2^n \subseteq K \subseteq R B_2^n$, oracle 
$\mathsf{WMEM}_K(\epsilon, \delta, R, r)$ is an $(\epsilon, \delta)$-weak membership oracle for $K$. 
Similarly, oracle $\mathsf{WVAL}_K(\epsilon, \delta, R, r)$ is an $(\epsilon, \delta)$-weak validity oracle.
Lemma~\ref{lem:centroid-scaling} along with the equivariance of $\Gamma$ (Lemma~\ref{lem:equivariance}) 
gives $(s_m/\sqrt{n}) B_2^\dim \subseteq \Gamma X \subseteq (s_M \sqrt{n}) B_2^\dim$.
Then $1/(\sqrt{\dim}s_M) B_2^\dim \subseteq (\Gamma X)^{\polar} \subseteq (\sqrt{\dim}/s_m) B_2^\dim$.
Set $r := 1/(\sqrt{\dim}s_M)$ and $R := \sqrt{\dim}/s_m)$.

\myparagraph{Detailed description of Subroutine~\ref{sub:cvar-oracle}.}
There are two main steps:
\begin{enumerate}
\item \label{step1} Use Subroutine~\ref{sub:polar-cvar-oracle} to create an $(\epsilon_2, \delta)$-weak membership oracle
$\mathsf{WMEM}_{(\Gamma X)^\polar}(\epsilon_2, \delta, R, r)$ for $(\Gamma X)^\polar$. 
Theorem~4.3.2 of \cite{GLS} (stated as Theorem~\ref{thm:wmem-to-wviol} here) is used in Lemma~\ref{lem:subroutine-2-correctness} to get an algorithm to implement an $(\epsilon_1, \delta)$-weak validity oracle $\mathsf{WVAL}_{(\Gamma X)^\polar}(\epsilon_1, \delta, R, r)$ running in
oracle polynomial time; $\mathsf{WVAL}_{(\Gamma X)^\polar} \allowbreak (\epsilon_1, \delta, R, r)$ invokes
$\mathsf{WMEM}_{(\Gamma X)^\polar}(\epsilon_2, \delta/Q, R, r)$ a polynomial number of times, specifically $Q=\poly(n, \log R)$ (see proof of Lemma~\ref{lem:subroutine-2-correctness}). 
The proof of Theorem~4.3.2 can be modified so that $\epsilon_2 \geq 1/\poly(1/\epsilon_1, R, 1/r)$. 

\item \label{step2} Lemma~4.4.1 of \cite{GLS} (stated as Lemma~\ref{lem:wviol-polar-wmem} here) gives an algorithm to construct an $(\epsilon, \delta)$-weak membership oracle
$\mathsf{WMEM}_{\Gamma X}(\epsilon, \delta, 1/r, 1/R)$ from $\mathsf{WVAL}_{(\Gamma X)^\polar}(\epsilon_1, \delta, R, r)$.
The proof of  Lemma~4.4.1 in \cite{GLS} shows
$\mathsf{WMEM}_{\Gamma X}(\epsilon, \delta, 1/r, 1/R)$ calls $\mathsf{WVAL}_{(\Gamma X)^\polar}(\epsilon_1, \delta, R, r)$ once, with
$\epsilon_1 \geq 1/ \poly(1/\epsilon, \norm{y}, 1/r)$ (where $y$ is the query point).
\details{Specifically, given $\eps$ and candidate point $y$, it is sufficient to call WVAL with $\eps_1 = \frac{1}{2} \frac{r \eps}{\norm{y} + r \norm{y} + r \eps}$}
\end{enumerate}

\begin{subroutine}[ht]
\caption{Weak Membership Oracle for $\Gamma X$}\label{sub:cvar-oracle}
\begin{algorithmic}[1]
\Require Query point $x \in \RR^\dim$,
samples from symmetric ICA model $X = AS$,
bounds $s_M \geq \sigma_{\max}(A)$, $s_m \leq \sigma_{\min}(A)$,
closeness parameter $\eps$,
failure probability $\delta$, 
access to a weak membership oracle for $(\Gamma X)^{\polar}$.
\Ensure $(\epsilon, \delta)$-weak membership decision for $x \in \Gamma X$.

\State Construct $\mathsf{WVAL}_{(\Gamma X)^\polar}(\epsilon_1, \delta, R, r)$ by invoking
$\mathsf{WMEM}_{(\Gamma X)^\polar}(\epsilon_2, \delta/Q, R, r)$. (See Step~\ref{step1} in the detailed description.)

\State Construct $\mathsf{WMEM}_{\Gamma X}(\epsilon, \delta, 1/r, 1/R)$ by invoking 
$\mathsf{WVAL}_{(\Gamma X)^\polar}(\epsilon_1, \delta, R, r)$. (See Step~\ref{step2} in the detailed description.)

\State Return the output of running $\mathsf{WMEM}_{\Gamma X}(\epsilon, \delta, 1/r, 1/R)$ on $x$. 
\end{algorithmic}
\end{subroutine}

\begin{lemma}[Correctness of Subroutine~\ref{sub:cvar-oracle}]\label{lem:subroutine-2-correctness}
Let $X=AS$ be given by a symmetric ICA model such that for all $i$ we have $\e (\abs{S_i}^{1+\gamma}) \leq M < \infty$ and normalized so that $\e \abs{S_i} = 1$.
Then, given a query point $x \in \RR^\dim$, $0<\eps \leq n^2$, $\delta > 0$, $s_M \geq \sigma_{\max}(A)$, and $s_m \leq \sigma_{\min}(A)$, Subroutine~\ref{sub:cvar-oracle} is an $\eps$-weak membership oracle for $x$ and $\Gamma X$ with probability $1-\delta$ using time and sample complexity
\(
\poly(n, M, 1/s_m, s_M, 1/\eps, 1/\delta).
\)\lnote{query time should also depend on query}
\end{lemma}


\begin{proof}
We first prove that $\mathsf{WVAL}_{(\Gamma X)^\polar}(\epsilon_1, \delta, R, r)$ 
(abbreviated to $\mathsf{WVAL}_{(\Gamma X)^\polar}$ hereafter) works correctly. To this end
we need to show that for any given input, $\mathsf{WVAL}_{(\Gamma X)^\polar}$ acts as an $\epsilon_1$-weak validity
oracle with probability at least $1-\delta$. Oracle $\mathsf{WVAL}_{(\Gamma X)^\polar}$ makes $Q$ queries 
to $\mathsf{WMEM}_{(\Gamma X)^\polar}(\epsilon_2, \delta/Q, R, r)$. If the answer to all these queries were correct then Theorem~4.3.2 from \cite{GLS} would apply
and would give that $\mathsf{WVAL}_{(\Gamma X)^\polar}$ outputs an answer as expected. Since these $Q$ queries are adaptive we
cannot directly apply the union bound to say that the probability of all of them being correct is at least
$1 - Q (\delta/Q) = 1-\delta$. However, a more careful bound allows us to do essentially that.


Let $q_1, \dotsc, q_k$ be the sequence of queries, where $q_i$ depends on the result of the previous queries.
For $i=1, \dotsc, k$, let $B_i$ be the event that the answer to query $q_i$ by Subroutine~\ref{sub:polar-cvar-oracle} is not correct according to the definition of the oracle it implements. 
These events are over the randomness of Subroutine~\ref{sub:polar-cvar-oracle} and event $B_i$ involves the randomness of $q_1, \dotsc, q_i$, as the queries could be adaptively chosen.
By the union bound, the probability that all answers are correct is at least $1-\sum_{i=1}^k \Pr(B_i)$. 
It is enough to show that $\Pr(B_i) \leq \delta/Q$.
To see this, we can condition on the randomness associated to $q_1, \dotsc q_{i-1}$. That makes $q_i$ deterministic, and the probability of failure is now just the probability that Subroutine~\ref{sub:polar-cvar-oracle} fails.
More precisely, $\Pr(B_i \giventhat q_1, \dotsc, q_{i-1}) \leq \delta/Q$, so that
\begin{align*}
\Pr(B_i) &= \int \Pr(B_i \giventhat q_1, \dotsc, q_{i-1}) \Pr(q_1, \dotsc, q_{i-1}) \, dq_1, \dotsc, dq_{i-1} \leq \delta/Q.
\end{align*}

This proves that the first step works correctly. Correctness of the second step follows directly because the
algorithm for construction of the oracle involves a single call to the input oracle as mentioned in 
Step~\ref{step2} of the detailed description. 

Finally, to prove that the running time of Subroutine~\ref{sub:cvar-oracle} is as claimed the main thing to note is that, as mentioned in Step~\ref{step1} of the detailed description, $\epsilon_2$ is polynomially small in $\epsilon_1$ and $\epsilon_1$ is polynomially small in $\epsilon$ and so $\epsilon_2$ is polynomially small in $\epsilon$.
\details{
	To see the parameters passed to the different oracles, let $y \in \QQ^\dim$ be the candidate point, $\eps > 0$ be the (rational) error given to the subroutine.
	We make a single call to the weak validity oracle with candidate point $c := y$, $\gamma := 1$, and error parameter \[\eps_{val} = \frac{1}{2} \frac{r \eps}{\norm{y} + r \norm{y} + r \eps}.\]

	We first define a weak separation oracle which, for a point $y \in \QQ^\dim$ two parameters $\eps_{sep}, \beta \in (0,1) \cap \QQ$, and convex body $(K; n, R, r, a_0)$, first calls the weak-membership with error $\eps_{mem} = \frac{r \eps_{sep}}{4R}$.
	Depending on the response, the separation algorithm either terminates or calls the weak-membership $n$ more times with error parameter \[ \eps_{mem} = C \frac{\beta^2 r^3 \eps_{sep}^2}{n^5 R^2 (R + r)} \]
	for an absolute constant $C$.

	The weak violation oracle works as follows (for parameters $c \in \QQ^n, \gamma, \eps_{val} \in \QQ$): in the \cite{GLS} proof of their Theorem~4.3.2, one defines a shallow-cut oracle which, for an ellipsoid $E(A,a)$ and matrix $Q$ such that $(n+1)^2 A^{-1} = QQ$, uses a single call to the separation oracle just described with parameter $y = Qa$, $\beta = (n+1)^{-1}$, and $\eps_{sep} = \min\{ (n+2)^{-2}, \eps_{val}/\norm{Q^{-1}} \}$.
	Using this shallow-cut oracle, one invokes the shallow-cut ellipsoid algorithm for the body $K' = K \cap \{x \suchthat c^T x \leq \gamma \}$ with radius $R$ and target volume $\eps_{vol} = ((r \eps_{val})/(2Rn))^2$.
	This call to the shallow-ellipsoid method takes time $n + \langle R \rangle + \langle \eps_{vol} \rangle$ where $\langle \cdot \rangle$ is the encoding length, unfortunately overloading our current use of inner product notation.
}
\end{proof}

\section{Orthogonalization via the uniform distribution in the centroid body}\label{sec:orthogonalization_uniform}

The following lemma says that linear equivariance allows orthogonalization: \nnote{something more 
meaningful needs to be said for introducing the following lemma.} \nnote{I moved the following lemma here which I think is the right place for it as it's here and only here that it gets used.}
\nnote{In the following lemma is that two different types of notations are being used for distributions, one using calligraphic letters and the other normal letters. This can be confusing.} \nnote{I suppose by calligraphic P sub S is meant the distribution of S. But this notation has not been
defined I think.}

\begin{lemma}\label{lem:orthogonalizer}
Let $U$ be a family of $n$-dimensional product distributions. Let $\bar U$ be the closure of $U$ under invertible linear transformations.
Let $Q(\measure)$ be an $n$-dimensional distribution defined as a function of $\measure \in \bar U$. Assume that $U$ and $Q$ satisfy:
\begin{enumerate}
\item\label{item:sym}
For all $\measure \in U$, $Q(\measure)$ is absolutely symmetric.
\item\label{item:equivariant} $Q$ is linear equivariant (that is, for any invertible linear transformation $T$ we have $Q(T\measure) = T Q(\measure)$).
\item\label{item:positive}
For any $\measure \in \bar U$, $\cov(Q(\measure))$ is positive definite.
\end{enumerate}
Then for any symmetric ICA model $X=AS$ with $\measure_S \in U$ we have $\cov(Q(\measure_X))^{-1/2}$ is an orthogonalizer of $X$.
\end{lemma}
\begin{proof}
Consider a symmetric ICA model $X=AS$ with $\measure_S \in U$.
Assumptions \ref{item:sym} and \ref{item:positive} imply $ D := \cov(Q(\measure_S))$ is diagonal and positive definite.
This with Assumption \ref{item:equivariant} gives $\cov(Q(\measure_X)) = \cov(A Q(\measure_S)) = A D A^T = A D^{1/2} (A D^{1/2})^T$.
Let $B = \cov(Q(\measure_X))^{-1/2}$ (the unique symmetric positive definite square root).
We have $B = R D^{-1/2} A^{-1}$ for some unitary matrix $R$ (see \cite[pg 406]{horn2012matrix}).
\details{To see this: $I = B \cov(Q(\measure_X)) B = B A D^{1/2} (B A D^{1/2})^T$, which implies $B A D^{1/2}$ is an unitary matrix, say, $R$. The claim follows.}
Thus, $B A = R D^{-1/2}$ has orthogonal columns, that is, it is an orthogonalizer for $X$.
\end{proof}

The following lemma applies the previous lemma to the special case when the distribution 
$Q(\measure)$ is the uniform distribution on $\Gamma \measure$. 

\begin{lemma}\label{lemma:uniform-orthogonalizer}
Let $X$ be a random vector drawn from a symmetric ICA model $X = AS$ such that for all $i$ we have $0 < \e \abs{S_i} < \infty $. Let Y be uniformly random in $\Gamma X$. Then $\cov(Y)^{-1/2}$ is an orthogonalizer of $X$.
\end{lemma}

\begin{proof}
We will use Lemma~\ref{lem:orthogonalizer}. 
After a scaling of each $S_i$, we can assume without loss of generality that $\e(S_i) = 1$. This will allow us to use Lemma~\ref{lem:centroid-scaling}.
Let $U$ = \{$\measure_W \suchthat \measure_W$ is an absolutely symmetric product distribution and $\e \abs{W_i} = 1 $, for all $i$\}. For $\measure \in \bar U$, let $Q(\measure)$ be the uniform distribution on the centroid body of $\measure$. \nnote{I changed ``centroid body'' to ``the uniform distribution on the centroid body.''}
For all $\measure_W \in U$, the symmetry of the $W_{i}$'s implies that $\Gamma \measure_W$, that is, $Q(\measure_W)$, is absolutely symmetric.
By the equivariance of $\Gamma$ (from Lemma~\ref{lem:equivariance}) and Lemma~\ref{lem:centroid-scaling} it follows that $Q$ is linear equivariant. Let $\measure \in \bar U$. Then there exist $A$ and $\measure_W \in U$ such that $\measure = A \measure_W$.  
So we get $\cov(Q(\measure)) = \cov(AQ(\measure_W)) = A \cov(\Gamma \measure_W) A^T$. \nnote{I am not sure what the purpose of
the last two sentences is.}
From Lemma~\ref{lem:centroid-scaling} we know $B_1^\dim \subseteq \Gamma \measure_W$
so that $\cov(\Gamma \measure_W)$ is a diagonal matrix with positive diagonal entries. 
This implies that $\cov(Q(\measure))$ is positive definite and thus by Lemma~\ref{lem:orthogonalizer}, $\cov(Y)^{-1/2}$ is an orthogonalizer of $X$. 
\end{proof}

\begin{algorithm}[ht]
\caption{Orthogonalization via the uniform distribution in the centroid body}\label{alg:orthogonalization_uniform}
\begin{algorithmic}[1]
\Require Samples from symmetric ICA model $X = AS$, bounds $s_M \geq \sigma_{\max}(A)$, $s_m \leq \sigma_{\min}(A)$, error parameters $\eps$ and $\delta$, access to an $(\epsilon, \delta)$-weak membership oracle for $\Gamma X$ provided by Subroutine~\ref{sub:cvar-oracle}.
\Ensure A matrix $B$ which orthogonalizes the independent components of $X$.
\State Let $\tilde \Sigma$ be an estimate of $\cov(\Gamma X)$ obtained
\iflong via Theorem~\ref{thm:covariance_estimation} \fi
\ifshort by generating uniform samples in $\Gamma X$ using Subroutine~\ref{sub:cvar-oracle} and \fi
sampling algorithm such as the one in \cite{DBLP:journals/jacm/DyerFK91}
with $\eps_c = \eps/(2 (\dim+1)^4)$, $r= s_m /\sqrt{\dim}$, $R = s_M \sqrt{\dim}$, and same $\delta$.
\State Return $B = \tilde{\Sigma}^{-1/2}$.
\end{algorithmic}
\end{algorithm}

\begin{theorem}[Correctness of algorithm \ref{alg:orthogonalization_uniform}] \label{thm:correctness_uniform_orth}
Let $X=AS$ be given by a symmetric ICA model such that for all $i$ we have $\e (\abs{S_i}^{1+\gamma}) \leq M < \infty$ and normalized so that $\e \abs{S_i} = 1$.
Then, given $0<\eps \leq \dim^2$, $\delta > 0$, $s_M \geq \sigma_{\max}(A)$ , $s_m \leq \sigma_{\min}(A)$
, Algorithm~\ref{alg:orthogonalization_uniform} outputs
a matrix $B$ so that
$ \norm{A^T B^T B A - D}_{2} \leq \eps$,
for a diagonal matrix $D$ with diagonal entries $d_1, \dotsc, d_\dim$ satisfying $1/(\dim+1)^2 \leq d_i \leq 1$.
with probability at least $1-\delta$ using $\poly_\gamma(n, M, 1/s_m, s_M,1/\eps, 1/\delta)$ time and sample complexity.\lnote{philosophical issue}
\details{current use of $1+\gamma$ Chebyshev imposes $1/\delta$ dependence. Might be improvable to $\log 1/\delta$ via something more involved such as median of means}
\end{theorem}
\begin{proof}
From Lemma~\ref{lem:centroid-scaling} we know $B_1^\dim \subseteq \Gamma S \subseteq [-1,1]^\dim$. 
Using the equivariance of $\Gamma$ (Lemma~\ref{lem:equivariance}), we get $\sigma_{\min}(A)/\sqrt{\dim} B_2^\dim \subseteq \Gamma X \subseteq \sigma_{\max}(A) \sqrt{\dim}B_2^\dim$.
Thus, to satisfy the roundness condition of  Theorem~\ref{thm:covariance_estimation}
we can take $r := s_m/\sqrt{\dim} \leq \sigma_{\min}(A)/\sqrt{\dim}$, $R := s_M \sqrt{\dim} \geq \sigma_{\max}(A) \sqrt{\dim}$.

Let $\tilde \Sigma$ be the estimate of $\Sigma := \cov(\Gamma X)$ computed by the algorithm.
Let $\tilde \Delta := A^{-1} \tilde \Sigma A^{-T}$ be the estimate of $\Delta := \cov{\Gamma S}$ obtained from $\tilde \Sigma$ according to how covariance matrices transform under invertible linear transformations of the underlying random vector.
As in the proof of Lemma~\ref{lem:orthogonalizer}, we have $\tilde \Sigma = A \tilde \Delta A^T$ and $B = R \tilde \Delta^{-1/2} A^{-1}$ for some unitary matrix $R$.
Thus, we have $A^T B^T B A = \tilde \Delta^{-1}$. 
It is natural then to set $D := \Delta^{-1} = \cov(\Gamma S)^{-1}$. Let $d_1, \dotsc, d_\dim$ be the diagonal entries of $D$.
We have, using Lemma~\ref{lem:inversion},
\begin{align}
\norm{A^T B^T B A - D}_2 
&= \norm{\tilde \Delta^{-1} - \Delta^{-1}}_2 \notag\\
&= \norm{\Delta^{-1}}_2 \frac{\norm{\Delta^{-1} (\tilde \Delta - \Delta)}_2}{1-\norm{\Delta^{-1} (\tilde \Delta - \Delta)}_2}.\label{equ:inversion}
\end{align}
As in \eqref{eq:inverse-stability}, we show that $\norm{\Delta^{-1} (\tilde \Delta - \Delta)}_2$ is small:

We first bound $(d_i)$, the diagonal entries of $D = \Delta^{-1}$.
Let $d_{\max} := \max_i d_i$ and $d_{\min} := \min_i d_i$. 
We find simple estimates of these quantities: 
We have $d_{\min} = 1/\norm{\Delta}_2$ and $\norm{\Delta}_2$ is the maximum variance of $\Gamma S$ along coordinate axes. 
From Lemma~\ref{lem:centroid-scaling} we know $\Gamma S \subseteq [-1,1]^\dim$, so that $\norm{\Delta}_2 \leq 1$ and $d_{\min} \geq 1$. 
Similarly, $d_{\max} = 1/\sigma_{\min}(\Delta)$, where $\sigma_{\min}(\Delta)$ is the smallest diagonal entry of $\Delta$. In other words, it is the minimum variance of $\Gamma S$ along coordinate axes. 
From Lemma~\ref{lem:centroid-scaling} we know $\Gamma S \supseteq B_1^\dim$, so that $\Gamma S \supseteq [-e_i, e_i]$ for all $i$ and Lemma~\ref{lem:minvariance} below implies $\sigma_{\min}(\Delta) \geq 1/(\dim+1)^2$. That is, $d_{\max} \leq (\dim+1)^2$.

From the bounds on $d_i$, Theorem~\ref{thm:covariance_estimation} and the fact discussed after it, we have
\begin{align*}
\norm{\Delta^{-1} (\tilde \Delta - \Delta)} 
\leq d_{\max} \norm{\Delta} \eps_c 
\leq (\dim+1)^2 \eps_c 
\leq 1/2,
\end{align*}
when $\eps_c \leq 1/(2(n+1)^2)$.\details{This together with the choice of $\eps_c$ below gives the assumption that $\eps <n^2$.}

This in \eqref{equ:inversion} with Theorem~\ref{thm:covariance_estimation} again gives
\begin{align*}
\norm{A^T B^T B A - D} 
\leq 2 \norm{\Delta^{-1}} \norm{\Delta^{-1} (\tilde \Delta - \Delta)} 
\leq 2 \norm{\Delta^{-1}}^2 \norm{\tilde \Delta - \Delta} 
\leq 2 d_{\max}^2 \eps_c \norm{\Delta} 
\leq 2 (\dim+1)^4 \eps_c.
\end{align*}
The claim follows by setting $\eps_c = \eps/(2 (\dim+1)^4)$.
The sample and time complexity of the algorithm comes from the calls to Subroutine~\ref{sub:cvar-oracle}. The 
number of calls is given by Theorem~\ref{thm:covariance_estimation}. This leads to the complexity as claimed.
\end{proof}

\iflong
\begin{lemma}\label{lem:minvariance}
Let $K \subseteq \RR^\dim$ be an absolutely symmetric convex body such that $K$ contains the segment $[-e_1, e_1] = \conv \{ e_1, -e_1\}$ (where $e_1$ is the first canonical vector). Let $X = (X_1, \dotsc, X_\dim)$ be uniformly random in $K$. Then $\var(X_1) \geq 1/(n+1)^2$.
\end{lemma}
\fi
\begin{proof}
Let $D$ be a diagonal linear transformation so that $DK$ isotropic. Let $d_{11}$ be the first entry of $D$.
It is known that any $\dim$-dimensional isotropic convex body is contained in the ball of radius $\dim+1$ \cite{MilmanPajor, sonnevend1989applications},\cite[Theorem 4.1]{KLS}.
Note that $D K$ contains the segment $[- d_{11} e_1, d_{11} e_1]$. This implies $d_{11} \leq (\dim+1)$. Also, by isotropy we have, $1 = \var(d_{11} X_1) = d_{11}^2 \var(X_1)$. The claim follows.
\end{proof}

\iflong
\section{Gaussian damping} \label{sec:gaussian_damping}
In this section we give an efficient algorithm for the heavy-tailed ICA problem when the ICA matrix is a
unitary matrix; 
no assumptions on the existence of moments of the $S_i$ will be required. 


The basic idea behind our algorithm is simple and intuitive: using $X$ we construct another ICA model $X_R = A S_R$, where $R > 0$ is a parameter which will be chosen later.
The components of $S_R$ have light-tailed distributions; in particular, all moments exist.
We show how to generate samples of $X_R$ efficiently using samples of $X$.
Using the new ICA model, the matrix $A$ can be estimated by applying existing ICA algorithms.

For a random variable $Z$ we will denote its the probability density function by $\rho_Z(\cdot)$.
The density of $X_R$ is obtained by multiplying the density of $X$ by a Gaussian damping factor.
More precisely,
\begin{align*}
\rho_{X_R}(x) \propto \rho_X(x) e^{-\norm{x}^2/R^2}.
\end{align*}
Define
\begin{align*}
K_{X_R} := \int_{\R^n} \rho_X(x) e^{-\norm{x}^2/R^2} \ud x,
\end{align*}
then
\begin{align*}
\rho_{X_R}(x) = \frac{1}{K_{X_R}} \rho_X(x) e^{-\norm{x}^2/R^2}.
\end{align*}

We will now find the density of $S_R$. Note that if $x$ is a value of $X_R$, and $s = A^{-1}x$ is the corresponding value of $S_R$,
then we have
\begin{align*}
\rho_{X_R}(x)
= \frac{1}{K_{X_R}} \rho_X(x) e^{-\norm{x}^2/R^2} 
= \frac{1}{K_{X_R}} \rho_S(s) e^{-\norm{As}^2/R^2}
= \frac{1}{K_{X_R}} \rho_S(s) e^{-\norm{s}^2/R^2} 
=: \rho_{S_R}(s),
\end{align*}
where we used that $A$ is a unitary matrix so that $\norm{As}=\norm{s}$. 
Also, $\rho_X(x) = \rho_S(s)$ follows from the change of variable formula and the fact that $\lrabs{\det A}=1$. 
We also used crucially the fact that the Gaussian distribution is spherically-symmetric.
We have now specified the new ICA model $X_R = A S_R$, and what remains is to show how to generate samples of $X_R$. 

\myparagraph{Rejection sampling.} Given access to samples from $\rho_X$ we will use rejection
sampling (see e.g. [Robert--Casella] \lnote{xxx}) to generate samples from $\rho_{X_R}$.

\begin{enumerate}
\item Generate $x \sim \rho_X$.
\item Generate $z \sim U[0,1]$.
\item If $z \in [0, e^{-\norm{x}^2/R^2}]$, output $x$; else, go to the first step.
\end{enumerate}

The probability of outputting a sample with a single trial in the above algorithm is $K_{X_R}$. Thus, the expected number of trials in the above algorithm for generating a sample is $1/K_{X_R}$.

We now choose $R$. 
There are two properties that we want $R$ sufficiently large so as to satisfy: (1) $K_{X_R} \geq C_1$ and
$\abs{\cum_4 S_{j, R}} \geq 1/n^{C_2}$ \jnote{Changed from $C$ to $C_2$} where $C_1 \in (0, 1/2)$ and $C_2 > 0$ are constants. Such a choice of $R$ exists and can be made efficiently; 
we outline this after the statement of Theorem~\ref{thm:ICA-orthogonal-damping}. \nnote{Actually, the statement about cumulants seems to require more care if some of the Si are not heavy-tailed}
Thus, the expected number of trials in rejection sampling before generating a sample is bounded above by
$1/C_1$. The lower bound on $K_{X_R}$ will also be useful in bounding the moments of the $S_{j,R}$, where 
$S_{j,R}$ is the random variable obtained by Gaussian damping of $S_j$ with parameter $R$, that is to say
\begin{align*}
\rho_{S_{j,R}}(x) \propto \rho_{S_j}(x) e^{-\norm{x}^2/R^2}`
\end{align*}


Define
\begin{align*}
K_{S_{j,R}} := \int_{\R}\rho_{s_j}(s_j) e^{-s_j^2/R^2} ds_j
\end{align*}
and let $K_{S_{{-j},R}}$ be the product of $K_{S_{k, R}}$ over $k \in [n]\setminus \{j\}$.
By $s_{-j} \in  \R^{n-1}$ we denote the vector $s \in \R^n$ with its $j$th element removed, then notice that
\begin{align} \label{eqn:KR-product}
K_{X_R} &= K_{S_R} = K_{S_{1, R}} K_{S_{2, R}} \ldots K_{S_{n, R}}, \\
K_{X_R} &= K_{S_{j, R}} K_{S_{{-j},S}}.
\end{align}

We can express the densities of individual components of $S_R$ as follows:
\begin{align*}
\rho_{S_{j,R}}(s_j) = \int_{\R^{n-1}} \rho_{S_R}(s) \, d s_{{-j}}
= \frac{1}{K_{X_R}} \int_{\R^{n-1}} \rho_S(s) e^{-\norm{s}^2/R^2} \, d s_{{-j}}
= \frac{K_{S_{{-j},R}} }{K_{X_R}} \rho_{S_{j}}(s_j) \, e^{-s_j^2/R^2}.
\end{align*}

This allows us to derive bounds on the moments of $S_{j,R}$:
\begin{align}
\E[S_{j,R}^4]
&= \frac{K_{S_{{-j},R}}}{K_{X_R}} \int_{\R} s_j^4 \, \rho_{S_j}(s_j) e^{-s_j^2/R^2} d s_j \nonumber \\
&\leq \frac{K_{S_{{-j},R}}}{K_{X_R}} \left(\max_{z \in \R} z^4 e^{-z^2/R^2}\right) \int_{\R} \rho_{S_j}(s_j) \, d s_j 
< \frac{K_{S_{{-j},R}}}{K_{X_R}} R^4 
\leq \frac{1}{K_{X_R}} R^4 \nonumber \\
&\leq \frac{1}{C_1} R^4. \label{eqn:damping-moment-bound}
\end{align}

We now state Theorem~4.2 from \cite{GVX} in a special case by setting parameters $k$ and $k_i$ in that theorem
to $4$ for $i \in [n]$. The algorithm analyzed in Theorem~4.2 of \cite{GVX} is called Fourier PCA. 
 \begin{theorem}\cite{GVX}\label{thm:ICA}
  Let $X \in \R^n$ be given by an ICA model $X=AS$ where $A \in \R^{ n
    \times n}$ is unitary 
  and the $S_i$ are mutually independent, $\E[S_i^4] \le M_4$ for some positive constant $M_4$,
and $\abs{\cum_{4}(S_i)} \ge \Delta$. For
  any $\epsilon > 0$ 
  with probability at least $1-\delta$, Fourier PCA will recover vectors $\{b_1, \ldots, b_n\}$
  such that there exist signs $\alpha_i = \pm 1$ and a permutation $\pi:[n]\to [n]$ satisfying
  \begin{align*}
    \norm{A_i - \alpha_i b_{\pi(i)}} \le \epsilon,
  \end{align*}
  using $\poly(n, M_4, 1/\Delta, 1/\epsilon, 1/\delta)$ samples. 
The running time of the algorithm is also of the same form. 
\end{theorem}

\nnote{To check: Is there an upper bound needed on epsilon in the hypothesis of the previous theorem?}

Combining the above theorem with Gaussian damping gives the following theorem. As previously noted, 
since we are doing rejection 
sampling in Gaussian damping, the expected number of trials to generate $N$ samples of $X_R$ is $N/K_{X_R}$. 
One can similarly prove high probability guarantees for the number of trials needed to generate $N$ samples.


\gaussiandamping*

We remark that the choice of $R$ in the above theorem can be made algorithmically in an efficient way.
Theorem~\ref{thm:cumulant-damping} below shows that as we increase $R$ the cumulant $\cum_4(S_{j,R})$ goes to infinty.
This shows that for any $\Delta>0$ there exists $R$ so as to satisfy the condition of the above theorem, namely
$\abs{\cum_{4}(S_{i,R})} \ge \Delta$. 
We now briefly
indicate how such an $R$ can be found efficiently (same sample and computational costs as in 
Theorem~\ref{thm:ICA-orthogonal-damping} above):
For a given $R$, we can certainly estimate $K_{X_R} = \int_{\R^n} \rho_X(x) e^{-\norm{x}^2/R^2} dx$ from samples, 
i.e. by the empirical mean
$\frac{1}{N}\sum_{i \in [N]} e^{-\norm{x^{(i)}}^2/R^2}$ of samples $x^{(1)}, \ldots, x^{(N)}$.\lnote{what's the distribution of the samples? is this really efficient?} 
This allows us to search for $R$ so that $K_{X_R}$ is as large as we want. 
This also gives us an upper bound on the fourth moment via Eq. \eqref{eqn:damping-moment-bound}.
To ensure that the fourth cumulants of all $S_{i,R}$ are large, note that for $a \in \R^n$ we have 
$\cum_4(a_1S_{1,R} + \dotsb + a_n S_{n,R}) = \sum_{i\in [n]} a_i^4 \cum_4(S_{i,R})$. 
We can estimate this quantity empirically, and minimize 
over $a$ on the unit sphere (the minimization can be done, e.g., using the algorithm in \cite{FJK}). 
This would give an estimate of $\min_{i \in [n]}\cum_4(S_{i,R})$
and allows us to search for an appropriate $R$.
\jnote{As seen in our experiments, it may not be this simple, since the optimization suggested would require access to samples of $S$, or some self-reinforcing procedure that first found approximate $S$, then estimated $R$ from these samples, finding then even more accurate samples of $S$, and so on.}

For the algorithm to be efficient, we also need $K_{X_R} \geq C_1$. This is easily achieved as we can empirically
estimate $K_{X_R}$ using the number of trials required in rejection sampling, and search for sufficiently large $R$ that makes the estimate sufficiently larger than $C_1$.

\nnote{
TODO: (1) Show for some example heavy-tailed distributions (e.g., Pareto) what the value of a good $R$ is in the 
above theorem. 
(2) \cite{GVX} assumes that the $S_j$ are all centered. This is not true in our application. We can center our r.v. with
respect to the empirical mean (this would require recomputing of the fourth moment etc. for the shifted $S_j$), but that's not
exact centering because of the use of the empirical mean instead of the actual mean. Better way is to use symmetrization.}

\section{The fourth cumulant of Gaussian damping of heavy-tailed distributions}
It is clear
that if r.v. $X$ is such that $\E (X^4) = \infty$ and $\E (X^2) < \infty$, then
$\cum_4(X_R) = \E (X_R^4) - 3 (\E (X_R^2))^2 \to \infty$ as $R \to \infty$. 
However, it does not seem clear when we have $\E (X^2) = \infty$ as well. We will show that in this case we also
get $\cum_4(X_R) \to \infty$ as $R \to \infty$.

We will confine our discussion to symmetric random variables for simplicity of exposition; for the purpose of
our application of the theorem this is w.l.o.g. by the argument in Sec.~\ref{sec:symmetrization}.

\begin{theorem}\label{thm:cumulant-damping}
Let $X$ be a symmetric real-valued random variable with $\E (X^4) = \infty$. 
Then $\cum_4(X_R) \to \infty$ as $R \to \infty$.
\end{theorem}
\begin{proof}
Fix a symmetric r.v. $X$ with $\E X^2 = \infty$; as previously noted, if $\E X^2 < \infty$ then the theorem is
easily seen to be true. 
Since $X$ is symmetric and we will be interested in the fourth cumulant, we can restrict our attention
to the positive part of $X$. So in the following we will actually assume that $X$ is a positive random variable. 
Fix $C > 100$ to be any large positive constant.
Fix a small positive constant $\eps_1 < \frac{1}{100 \, C}$. Also fix another small positive constant $\eps_2 < 1/10$.
Then there exists 
$a > 0$ such that
\begin{align} \label{eqn:a_eps1}
\Pr[X \geq a] =  \int_{a}^\infty \rho_X(x) dx \leq \eps_1.
\end{align}
\jnote{Equals and inequality were interchanged. Is this what was really intended?}

Let $\tilde{m}_2(R) := \int_0^\infty x^2\rho_X(x) e^{-x^2/R^2} dx$. 
Recall that $K_{X_R} = \int_{0}^{\infty} e^{-x^2/R^2} \rho(x) dx = \e e^{-X^2/R^2}$.
Note that if $R \geq a$ (which we assume in the sequel), then
\begin{align} \label{eqn:K_{X_R}}
1 > K_{X_R} > \frac{1-\eps_1}{e}. 
\end{align}
\details{For the second inequality, in the integral for $K_{X_R}$ restrict it to $x \in [0, a]$, now note that in this range the integrand $e^{-x^2/R^2}$ is $\geq 1/e$ (as in the previous like we assumed $R \geq a$), and the probability mass in this range is $1-\eps_1$ (by (16)).}
Since $\tilde{m}_2(R) \to \infty$ as $R \to \infty$, by choosing $R$ sufficiently large we can ensure that
\begin{align} \label{eqn:b_eps2damping}
\int_a^\infty x^2 \rho_X(x) e^{-x^2/R^2} dx \geq (1-\eps_2) \int_0^\infty x^2 \rho_X(x) e^{-x^2/R^2} dx.
\end{align}
 \details{ to see clearly, split the right side and combine non-constant terms on the left, they will go to infinity with $R$}
Moreover, we choose $R$ to be sufficiently large so that $\sqrt{C\tilde{m}_2(R)} > a$.
Then
\begin{align*}
\tilde{m}_2(R)
&= \int_0^\infty x^2 \rho_X(x) e^{-x^2/R^2} dx \\
&= \int_0^a x^2 \rho_X(x) e^{-x^2/R^2} dx + \int_a^{\sqrt{C\tilde{m}_2(R)}} x^2 \rho_X(x) e^{-x^2/R^2} dx + \int_{\sqrt{C\tilde{m}_2(R)}}^\infty x^2 \rho_X(x) e^{-x^2/R^2} dx \\
&\leq \eps_2 \int_0^\infty x^2 \rho_X(x) e^{-x^2/R^2} dx +  \int_a^{\sqrt{C\tilde{m}_2(R)}} x^2 \rho_X(x) e^{-x^2/R^2} dx + \int_{\sqrt{C\tilde{m}_2(R)}}^\infty x^2 \rho_X(x) e^{-x^2/R^2} dx
\;\; \\ 
&\leq \eps_2\, \tilde{m}_2(R) + \eps_1 C \, \tilde{m}_2(R) + \int_{\sqrt{C\tilde{m}_2(R)}}^\infty x^2 \rho_X(x) e^{-x^2/R^2} dx \;\; \text{(by \eqref{eqn:a_eps1})} \\
&= (C\eps_1+\eps_2) \, \tilde{m}_2(R) + \int_{\sqrt{C \tilde{m}_2(R)}}^\infty x^2 \rho_X(x) e^{-x^2/R^2} dx.
\end{align*}
Summarizing the previous sequence of inequalities:
\begin{align} \label{eqn:estimate_second2}
\int_{\sqrt{C \tilde{m}_2(R)}}^\infty x^2 \rho_X(x) e^{-x^2/R^2} dx \geq (1- C\eps_1 - \eps_2) \, \tilde{m}_2(R).
\end{align}

Now
\begin{align*}
\E X_R^4 > K_{X_R} \, \E X_R^4 &= \int_0^\infty x^4 \rho_X(x) e^{-x^2/R^2} dx \;\;\text{(the inequality uses \eqref{eqn:K_{X_R}})}\\
&\geq \int_{\sqrt{C \tilde{m}_2(R)}}^\infty x^4 \rho_X(x) e^{-x^2/R^2} dx \\
&\geq C \tilde{m}_2(R) \int_{\sqrt{C \tilde{m}_2(R)}}^\infty x^2 \rho_X(x) e^{-x^2/R^2} dx \\
&\geq C \tilde{m}_2(R) (1- C\eps_1 - \eps_2) \, \tilde{m}_2(R) \;\;\text{(by \eqref{eqn:estimate_second2})} \\
&= C (1- C\eps_1 - \eps_2) \, \tilde{m}_2(R)^2 \\
&= C (1- C\eps_1 - \eps_2) \, (K_{X_R} \, \E X_R^2)^2 \\
&\geq  C (1- C\eps_1 - \eps_2) \left(\frac{1-\eps_1}{e}\right)^2 (\E X_R^2)^2 \;\;\text{(by \eqref{eqn:K_{X_R}})}.
\end{align*}

Now note that $C (1- C\eps_1 - \eps_2) \left(\frac{1-\eps_1}{e}\right)^2 > 10$ for our choice of the parameters. Thus
$\cum_4(X_R) = \E X_R^4 - 3 (\E X_R^2)^2 > 7 (\E X_R^2)^2$, and by our assumption
$\E X_R^2 \to \infty$ with $R \to \infty$.
\end{proof} 
\fi

\section{Symmetrization} \label{sec:symmetrization}
As usual we work with the ICA model $X = AS$. 
Suppose that we have an ICA algorithm that works when each of the component random variable $S_i$ 
is symmetric, i.e. its probability density function satisfies $\phi_i(y) = \phi_i(-y)$ for all $y$, 
with a polynomial dependence on the upper bound $M_4$ on the fourth moment of the $S_i$ and inverse
polynomial dependence on the lower bound $\Delta$ on the fourth cumulants of the $S_i$. Then we show
that we also have an algorithm without the symmetry assumption and with a similar dependence on $M_4$
and $\Delta$. 
We show that without loss of generality we may restrict our attention to symmetric densities, 
i.e. we can assume that each
of the $S_i$ has density function satisfying $\phi_i(y) = \phi_i(-y)$. To this end, let $S'_i$ be an 
independent copy of $S_i$ and set $\bar{S}_i := S_i - S'_i$. Similarly, let $\bar{X}_i := X_i - X'_i$. 
Clearly, the $S_i$ and $X_i$ have symmetric densities.
The new random variables still satisfy the ICA model: $\bar{X}_i = A \bar{S}_i$. 
Moreover, the moments and cumulants of the $\bar{S}_i$ behave similarly to those of the $S_i$: 
For the fourth moment, assuming it exists, we have
\( \E[\bar{S}_i^4] = \E[(S_i-S'_i)^4] \leq 2^4 \E[S_i^4]\).
The inequality above can easily be proved using the binomial expansion and H\"older's inequality:
\begin{align*}
\E[(S_i-S'_i)^4] &= \E[S_i^4] + 4 \E[S_i^3]\,\E[S'_i] + 6 \E[S_i^2]\,\E[(S'_i)^2] + 4 \E[S_i]\,\E[(S'_i)^3] + 
\E[(S'_i)^4] \leq 16 \, \E[S_i^4].
\end{align*}
The final inequality follows from the fact that that for each term in the LHS, e.g. $\E[S_i^3]\,\E[S'_i]$ we have 
$\E[S_i^3]\,\E[S'_i] \leq \E[S_i^4]^{3/4} \E[(S'_i)^4]^{1/4} = \E[S_i^4]$.

For the fourth cumulant, again assuming its existence, we have 
\( \cum_4(\bar{S}_i) = \cum_4(S_i - S'_i) = \cum_4(S_i)+\cum_4(-S_i) = 2\, \cum_4(S_i) \).

Thus if the the fourth cumulant of $S_i$ is away from $0$ then so is the fourth cumulant of $\bar{S}_i$.

\iflong
\section{Putting things together}

In this section we combine the orthogonalization procedure (Algorithm~\ref{alg:orthogonalization_uniform}
with performance guarantees in Theorem~\ref{thm:correctness_uniform_orth}) 
with ICA for unitary $A$ via Gaussian damping to prove our main theorem, Theorem \ref{thm:putting_together}.

\main*

As noted in the introduction, intuitively, $R$ in the theorem statement above  measures how large a ball we need to restrict the 
distribution to so that there is at least a constant (or $1/\poly(n)$ if needed) probability mass in it 
and moreover each $S_i$ when restricted to the interval $[-R, R]$ has the fourth cumulant at least 
$\Omega(\Delta)$. 
Formally, $R > 0$ is such that $\int_{\R^n} \rho_{\hX}(x) e^{-\norm{x}^2/R^2} \ud x \geq p(n) > 0$, where 
$1/\poly(n) < p(n) < 1$ can be chosen, and for simplicity, we will fix to $1/2$. Moreover, $R$ satisfies that 
$\cum_4(S_{i,R}) \geq \Omega(\Delta/n^4)$ for all $u \in S^{n-1}$ and $i \in [n]$, where $S_{i,R}$ is the Gaussian damping
with parameter $R$ of $S_i$.

In Sec.~\ref{sec:symmetrization} we saw that the moments and cumulants of the non-symmetric random variable behave similarly to those of the symmetric random variable.
So by the argument of Sec.~\ref{sec:symmetrization} we assume that our ICA model is symmetric.
Theorem~\ref{thm:correctness_uniform_orth} shows that Algorithm~\ref{alg:orthogonalization_uniform} gives 
us a new ICA model with the ICA matrix having approximately orthogonal columns. 
We will apply Gaussian damping to this new ICA model. 
In Theorem~\ref{thm:correctness_uniform_orth}, it was convenient to use the normalization $\E\abs{S_i} = 1$ for all $i$. 
But for the next step of Gaussian damping we will use a different normalization, namely, the columns of the ICA matrix have unit length. 
This will require us to rescale (in the analysis, not algorithmically) the components of $S$ appropriately as we now describe. 

Algorithm~\ref{alg:orthogonalization_uniform} provides us with a matrix $B$ such that the columns of 
$C = BA$ are approximately orthogonal: $C^T C \approx D$ where $D$ is a diagonal matrix. 
Thus, we can rewrite our ICA model as $Y = C S$, where $Y = BX$. 
We rescale $C_i$, the $i$th column of $C$, 
by multiplying it by $1/\norm{C_i}$. Denoting by $L$ the diagonal matrix with the $i$th diagonal entry 
$1/\norm{C_i}$, the matrix obtained after the above rescaling of $C$ is $CL$ and we have 
$(CL)^T CL \approx I$. We can again rewrite our ICA model as $Y = (CL) (L^{-1}S)$. Setting $\hE := CL$ and $T := L^{-1}S$ we can rewrite our ICA model as $\hat{Y} = \hE T$. This is the model we will plug into the Gaussian damping procedure. Had Algorithm~\ref{alg:orthogonalization_uniform} provided us with 
perfect orthogonalizer $B$ (so that $C^T C = D$) we would obtain a model $Y = E T$ where $E$ is 
unitary. We do get however that $\hE \approx E$. 
To continue with a more standard ICA notation, from here on we will write $\hX = \hE S$ for $\hat{Y} = \hE T$
and $X = ES$ for $Y = E T$. 

Applying Gaussian damping to $X = ES$ gives us a new ICA model $X_R = E S_R$ as we saw in 
Sec.~\ref{sec:gaussian_damping}. 
But the model we have access to is $\hX = \hE S$. 
We will apply Gaussian damping to it to get the r.v. $\hX_R$. 
 Formally, ${\hX}_R$ is defined starting with the model $\hX = \hE S$
just as we defined $X_R$ starting with the model $X = ES$ (recall that for a random variable $Z$, we denote
its probability density function by $\rho_Z(\cdot)$):
\begin{align*}
\rho_{{\hX}_R}(x) := \frac{1}{K_{{\hX}_R}} \rho_{\hX}(x) e^{-\norm{x}^2/R^2},
\end{align*}
where $K_{{\hX}_R} := \int_{\R^n} \rho_{\hX}(x) e^{-\norm{x}^2/R^2} \ud x$. The parameter $R$ has been chosen so that
$K_{{\hX}_R} > C_1 := 1/2$ and $\cum_4\,\inner{\hX_R}{u} \geq \Delta$ for all $u \in \S^{n-1}$. By the discussion 
after Theorem~\ref{thm:ICA-orthogonal-damping} (the restatement in Sec.~\ref{sec:gaussian_damping}), 
this choice of $R$ can be made efficiently. (The discussion there is in terms of the directional moments of
$S$, but note that the directional moments of $X$ also give us directional moments of $S$. We omit further 
details.)
But now since the matrix $\hE$ in our ICA model is only approximately unitary, after applying Gaussian damping the obtained random variable ${\hX}_R$ is not given by an ICA model (in particular, it may not have independent coordinates in any basis), although it is close to $X_R$ in a sense to be made precise soon.

Because of this, Theorem~\ref{thm:ICA} is not directly usable for plugging in the samples of ${\hX}_R$. 
To address this discrepancy we will need a robust version of Theorem~\ref{thm:ICA} which also requires us
to specify in a precise sense that ${\hX}_R$ and $X_R$ are close. To this end,
we need some standard terminology from probability theory.
The characteristic function of r.v. $X \in \R^n$ is defined to be $\phi_X(u) = \E (e^{i u^TX})$, 
where $u \in \R^n$.
The cumulant generating function, also known as the second characteristic function,
is defined by $\psi_X(u) = \log \phi_X(u)$. 
The algorithm in \cite{GVX} estimates the second derivative of $\psi_X(u)$ and computes its eigendecomposition. 
(In \cite{Yeredor} and \cite{GVX}, 
this second derivative is interpreted as a kind of covariance matrix of $X$ but with the twist that 
a certain ``Fourier'' weight is used in the expectation computation for the covariance matrix. We will 
not use this interpretation here.\lnote{is this comment relevant? remove?}) 
Set $\Psi_{X}(u) := D^2 \psi_X(u)$, the Hessian matrix of $\psi_X(u)$. We can now state the robust 
version of Theorem~\ref{thm:ICA}.
\nnote{maybe mention somewhere here that we do not have to use GVX and could also use other papers}

\begin{theorem}\label{thm:ICA-robust}
Let $X$ be an $\dim$-dimensional random vector given by an ICA model $X=AS$ where $A \in \R^{ n \times n}$ is unitary 
  and the $S_i$ are mutually independent, $\E[S_i^4] \le M_4$
and $\abs{\cum_{4}(S_i)} \ge \Delta$ for positive constants $M_4$ and $\Delta$. 
Also let $\eps_{\ref{thm:ICA-robust}} \in [0,1]$.  
Suppose that we have another random variable $\hX$ that is close to $X$ in the following sense:
\begin{align*}
\lrabs{\Psi_{\hX}(u)-\Psi_X(u)} \leq \eps_{\ref{thm:ICA-robust}},
\end{align*}
for any $u \in \R^n$ with $\norm{u} \leq 1$.
Moreover, $\E[\inner{X}{u}^4] \le M_4$ for $\norm{u} \leq 1$.
When Fourier PCA is given samples of $\hX$ it will recover 
vectors $\{b_1, \ldots, b_n\}$
  such that there exist signs $\alpha_i \in \{-1, 1\}$ and a permutation $\pi:[n]\to [n]$ satisfying
  \begin{align*}
    \norm{A_i - \alpha_i b_{\pi(i)}} \le \epsilon_{\ref{thm:ICA-robust}} \left(\frac{M_4}{\delta \Delta}\right)^5,
  \end{align*}
  in $\poly(n, M_4, 1/\Delta, 1/\epsilon_{\ref{thm:ICA-robust}}, 1/\delta_{\ref{thm:ICA-robust}})$ samples and time complexity and 
 with probability at least $1-\delta_{\ref{thm:ICA-robust}}$.
\end{theorem}
While this theorem is not stated in \cite{GVX}, it is easy to derive from their proof of 
Theorem~\ref{thm:ICA}; we now briefly sketch the proof of Theorem~\ref{thm:ICA-robust} indicating the changes
one needs to make to the proof of Theorem~\ref{thm:ICA} in \cite{GVX}. 

\begin{proof}
Ideally, for input model $X=AS$ with $A$ unitary, algorithm Fourier PCA  would proceed by diagonalizing 
$\Psi_X(u)$. But it can only compute an approximation 
$\tilde{\Psi}_X(u)$ which is the empirical estimate
for $\Psi_X(u)$. For all $u$ with $\norm{u} \leq 1$, it is shown that with high probability we have
\begin{align} \label{eqn:empirical-sigma}
\norm{\tilde{\Psi}_X(u)-\Psi_X(u)}_F<\epsilon.
\end{align}
Then, a matrix perturbation 
argument is invoked to show that if the diagonalization procedure used in Fourier PCA is applied to 
$\tilde{\Psi}_X(u)$ instead of $\Psi_X(u)$, 
one still recovers a good approximation of $A$. This previous step uses a random $u$ chosen from a Gaussian 
distribution so that the eigenvalues of $\Psi_X(u)$ are sufficiently spaced apart for the eigenvectors to be 
recoverable (the assumptions on the distribution ensure that the requirement of $\norm{u} \leq 1$ 
is satisfied with high probability).
The only property of $\tilde{\Psi}_X(u)$ used in this argument is
\eqref{eqn:empirical-sigma}. To prove Theorem~\ref{thm:ICA-robust}, we show that the estimate 
$\tilde{\Psi}_{\hX}$ is also good: 
\begin{align*}
\norm{\tilde{\Psi}_{\hX}(u) - \Psi_X(u)}_F 
< \norm{\tilde{\Psi}_{\hX}(u) - \Psi_{\hX}(u)}_F + \norm{\Psi_{\hX}(u) - \Psi_X(u)}_F 
< 2\epsilon,
\end{align*}
where we ensured that $\norm{\tilde{\Psi}_{\hX}(u) - \Psi_{\hX}(u)}_F < \epsilon$ by taking sufficiently many samples of $\hat{X}$ to get a good estimate with probability at least $\delta$; as in \cite{GVX}, 
a standard concentration argument shows that 
$\poly(n, M_4, 1/\Delta, 1/\epsilon, 1/\delta)$ samples suffice for this purpose. 
Thus the diagonalization procedure can be applied to $\tilde{\Psi}_{\hX}(u)$.
The upper bound of $2\epsilon$ above
translates into error 
$< \epsilon \left(\frac{M_4}{\delta \Delta}\right)^5$ in the final recovery guarantee, with the extra factor coming from the eigenvalue gaps of $ \Psi_{X}(u)$.\lnote{Too sketchy} 
\end{proof}

To apply Theorem~\ref{thm:ICA-robust} to our situation, we need
\begin{align} \label{eqn:Sigma_estimate}
\tilde{\Psi}_{\hat{X}_R}(u) \approx \Psi_{X_R}(u).
\end{align}
This will follow from the next lemma. 

Note that
\begin{align} \label{eqn:expansion-Sigma}
\Psi_{X}(u) = D^2 \psi_X(u) = \frac{D^2 \phi_X(u)}{\phi_X(u)} - \frac{(D \phi_X(u))^T(D \phi_X(u))}{\phi_X(u)^2}.
\end{align}
(The gradient $D \phi_X(u)$ is a row vector.)
Thus, to show \eqref{eqn:Sigma_estimate} it suffices to show that each expression on the RHS of the previous 
equation is appropriately close: 

\begin{lemma} \label{lem:Sigma-estimates}
Let $\lambda \in [0,1/2]$, and let $E, \hE \in \R^{n \times n}$ such that $E$ is unitary and 
$\norm{E-\hE}_2 \leq \lambda^2/3$. Let $X_R$ and ${\hX}_R$ be the random variables obtained by 
applying Gaussian damping to the ICA models $X=ES$ and $\hX = \hE S$, resp.  
Then, for $\norm{u} \leq 1$, we have 
\begin{align*}
\lrabs{\phi_{X_R}(u) - \phi_{\hX_R}(u)} &\leq R \lambda^2/3 + 4\lambda + \frac{4\lambda}{K_{X_R}}, \\
\norm{D \phi_{X_R}(u) - D \phi_{\hX_R}(u)} &\leq O(n\lambda R^2), \\
\norm{D^2\phi_{X_R}(u) - D^2\phi_{\hX_R}(u)}_F &\leq O(n^2\lambda R^3).
\end{align*}
\end{lemma}

\begin{proof}
We will only prove the first inequality; proofs of the other two are very similar and will be omitted. 
In the second equality in the displayed equations below we use that 
$\int_{\R^n}e^{iu^Tx} e^{-\norm{x}^2/R^2}\rho_{\hX}(x)\, dx = \int_{\R^n}e^{iu^T\hE s} e^{-\norm{\hE s}^2/R^2}\rho_{S}(s) \, ds$.
One way to see this is to think of the two integrals as 
expectations: $\E\left(e^{iu^T\hX} e^{-\norm{\hX}^2/R^2}\right) = \E\left( e^{iu^T\hE S} e^{-\norm{\hE S}^2/R^2}\right)$.
\begin{align}
\lrabs{\phi_{X_R}(u) - \phi_{\hX_R}(u)} \nonumber 
&= \lrabs{ \frac{1}{K_{X_R}}\int_{\R^n}e^{iu^Tx} e^{-\norm{x}^2/R^2}\rho_X(x) dx
- \frac{1}{K_{\hX_R}}\int_{\R^n}e^{iu^Tx} e^{-\norm{x}^2/R^2}\rho_{\hX}(x) dx} \nonumber \\
&= \lrabs{ \frac{1}{K_{X_R}}\int_{\R^n}e^{iu^TEs} e^{-\norm{Es}^2/R^2}\rho_S(s) ds
- \frac{1}{K_{\hX_R}}\int_{\R^n}e^{iu^T\hE s} e^{-\norm{\hE s}^2/R^2}\rho_{S}(s) ds} \nonumber \\
&\leq \frac{1}{K_{X_R}} \lrabs{ \int_{\R^n} e^{iu^TEs} e^{-\norm{Es}^2/R^2}\rho_S(s) ds 
- \int_{\R^n}e^{iu^T\hE s} e^{-\norm{\hE s}^2/R^2}\rho_{S}(s) ds} \nonumber \\
& + \lrabs{\frac{1}{K_{X_R}} - \frac{1}{K_{\hX_R}}} \cdot \lrabs{\int_{\R^n}e^{iu^T\hE s} e^{-\norm{\hE s}^2/R^2}\rho_{S}(s) ds} \nonumber \\
&\leq \frac{1}{K_{X_R}} \underbrace{\int_{\R^n} \lrabs{e^{iu^TEs} e^{-\norm{Es}^2/R^2} - e^{iu^T\hE s} e^{-\norm{\hE s}^2/R^2}}\rho_{S}(s) ds}_{G} 
+  \underbrace{\lrabs{\frac{1}{K_{X_R}} - \frac{1}{K_{\hX_R}}} K_{\hX_R}}_{H}. \nonumber 
\end{align}

Now 
\begin{align*}
G = \int_{\R^n} \lrabs{ e^{iu^T(\hE-E) s} e^{(\norm{E s}^2 - \norm{\hE s}^2)/R^2} - 1} e^{-\norm{Es}^2/R^2} \rho_S(s) ds.
\end{align*}

We have 
\begin{align*}
\lrabs{ e^{iu^T(\hE-E) s} e^{(\norm{E s}^2 - \norm{\hE s}^2)/R^2} - 1}
\leq \lrabs{ e^{iu^T(\hE-E) s} -1} + \lrabs{e^{(\norm{E s}^2 - \norm{\hE s}^2)/R^2} - 1},
\end{align*}

and so 
\begin{align} \label{eqn:G-decomposition}
G \leq \int_{\R^n} \lrabs{ e^{iu^T(\hE-E) s} -1}\, e^{-\norm{Es}^2/R^2} \rho_S(s) ds + 
\int_{\R^n} \lrabs{e^{(\norm{E s}^2 - \norm{\hE s}^2)/R^2} - 1}\, e^{-\norm{Es}^2/R^2} \rho_S(s) ds.
\end{align}

For the first summand in \eqref{eqn:G-decomposition} note that
\begin{align*}
\lrabs{ e^{iu^T(\hE-E) s} -1} \leq \lrabs{ u^T (\hat{E}-E) s}.
\end{align*}
This follows from the fact that for real $\theta$ we have 
\begin{align*}
\lrabs{e^{i \theta}-1}^2 = (\cos\theta -1)^2 + \sin^2\theta = 2 - 2\cos\theta= 4\sin^2(\theta/2) \leq \theta^2.
\end{align*}

So, 
\begin{align*}
&\int_{\R^n} \lrabs{u^T (\hat{E}-E) s} e^{-\norm{Es}^2/R^2} \rho_S(s) ds \\
& \leq \norm{u} \norm{\hat{E}-E}_2 \int_{\R^n} \norm{s} e^{-\norm{s}^2/R^2} \rho_S(s) ds  
\;\;\;\text{(using $\norm{Es}=\norm{s}$ as $E$ is unitary)} \\
& \leq \norm{u} \norm{\hat{E}-E}_2 \left(\max_{z \in \R} z e^{-z^2/R^2}  \right) \int_{\R^n} \rho_S(s) ds  \\
& \leq \norm{u} \norm{\hat{E}-E}_2 R  \\
&\leq R \lambda^2/3,
\end{align*}
where the last inequality used our assumption that $\norm{u} \leq 1$. 

We will next bound second summand in \eqref{eqn:G-decomposition}.
Note that $\norm{Es} = \norm{s}$ and 
$\norm{Es}^2 - \norm{\hE s}^2 \leq (\norm{E}+\norm{\hE})(\norm{E-\hat{E}})\norm{s}^2$. 
Since $\norm{E-\hat{E}} \leq \lambda^2/3 << 1$, we get 
$\norm{Es}^2 - \norm{\hE s}^2 \leq (\norm{E}+\norm{\hE})(\norm{E-\hat{E}}) \norm{s}^2 \leq \lambda^2 \norm{s}^2$. We will use that $e^\lambda - 1 < \lambda + \lambda^2$
for $\lambda \in [0,1/2]$ which is satisfied by our assumption. 
Now the second summand in \eqref{eqn:G-decomposition} can be bounded as follows. 

\begin{align*}
& \int_{\R^n} \lrabs{e^{(\norm{E s}^2 - \norm{\hE s}^2)/R^2} - 1} e^{-\norm{Es}^2/R^2} \rho_S(s) ds \\
&\leq  \int_{\R^n} \lrabs{e^{\lambda^2 \norm{s}^2/R^2}-1} e^{-\norm{s}^2/R^2} \rho_S(s) ds \\
&\leq  \int_{\R^n} \lrabs{e^{\lambda^2 \norm{s}^2/R^2}-1} e^{-\norm{s}^2/R^2} \rho_S(s) ds \\
&=  \int_{\norm{s} \leq R/\sqrt{\lambda}} \lrabs{e^{\lambda^2 \norm{s}^2/R^2}-1} e^{-\norm{s}^2/R^2} \rho_S(s) ds 
+  \int_{\norm{s} > R/\sqrt{\lambda}} \lrabs{e^{\lambda^2 \norm{s}^2/R^2}-1} e^{-\norm{s}^2/R^2} \rho_S(s) ds \\
&\leq  \int_{\norm{s} \leq R/\sqrt{\lambda}} (\lambda+\lambda^2) e^{-\norm{s}^2/R^2} \rho_S(s) ds 
+  \int_{\norm{s} > R/\sqrt{\lambda}} e^{-(1-\lambda^2)\norm{s}^2/R^2} \rho_S(s) ds \\
&\leq (\lambda+\lambda^2) K_{X_R} +  e^{-(1-\lambda^2)/\lambda} \\
&\leq 2\lambda K_{X_R} + 2\lambda \;\;\;\text{(using $\lambda < 1/2$)}. 
\end{align*}

Combining our estimates gives
\begin{align*}
G \leq \frac{R \lambda^2}{3 K_{X_R}} + 2\lambda + \frac{2\lambda}{K_{X_R}}.
\end{align*}

Finally, to bound $H$, note that 
\begin{align*}
\frac{1}{K_{X_R}}\lrabs{K_{X_R}-K_{\hX_R}} &= \frac{1}{K_{X_R}} \lrabs{ \int_{\R^n} e^{-\norm{x}^2/R^2}\rho_X(x) dx
- \int_{\R^n} e^{-\norm{x}^2/R^2}\rho_{\hX}(x) dx} \\
&= \frac{1}{K_{X_R}} \lrabs{ \int_{\R^n} e^{-\norm{Es}^2/R^2}\rho_S(s) ds - \int_{\R^n} e^{-\norm{\hE s}^2/R^2}\rho_{S}(s) ds}.
\end{align*} 
This we just upper-bounded above by $2\lambda + \frac{2\lambda}{K_{X_R}}$. 

Thus we have the final estimate
\begin{align*}
&\lrabs{\phi_{X_R}(u) - \phi_{\hX_R}(u)} \leq G + H \leq R \lambda^2/3 + 4\lambda + \frac{4\lambda}{K_{X_R}}.
\end{align*}

The proofs of the other two upper bounds in the lemma follow the same general pattern with slight changes.
\end{proof}

We are now ready to prove Theorem~\ref{thm:putting_together}.
\begin{proof}[Proof of Theorem~\ref{thm:putting_together}]
We continue with the context set after the statement of Theorem~\ref{thm:putting_together}. The plan 
is to apply Theorem~\ref{thm:ICA-robust} to ${\hX}_R$ and $X_R$. To this end we begin by showing that the premise
of Theorem~\ref{thm:ICA-robust} is satisfied.

Theorem~\ref{thm:correctness_uniform_orth}, with $\delta_{\ref{thm:correctness_uniform_orth}}=\delta/2$ and
$\epsilon_{\ref{thm:correctness_uniform_orth}}$ to be specified later,
provides us with a matrix $B$ such that the columns of $BA$ are approximately orthogonal:
$\norm{(BA)^TBA - D}_2 \leq \epsilon_{\ref{thm:correctness_uniform_orth}}$ for some diagonal matrix $D$. 
Now we set $\hE := B A L$, where $L = \diag(L_1, \dots, L_\dim) = \diag(1/\sqrt{d_1}, \ldots, 1/\sqrt{d_n})$. Theorem~\ref{thm:correctness_uniform_orth}
implies that
\begin{align} \label{eqn:L_i}
1 \leq L_i \leq (n+1).
\end{align}
Then
\begin{align*}
\norm{\hE^T\hE - I}_2 &= \norm{(BAL)^T(BAL) - I}_2 \\
&= \norm{L^T(BA)^TBA L - L^TDL}_2
\leq \norm{L}_2^2 \norm{(BA)^TBA - D}_2 \leq (n+1)^2\epsilon_{\ref{thm:correctness_uniform_orth}},
\end{align*}
because $L_i \leq (n+1)$ by \eqref{eqn:L_i}.
For $\hE$ as above, there exists a unitary $E$ such that 
\begin{align} \label{eqn:EhE}
\norm{E-\hE}_2 \leq (n+1)^2 \epsilon_{\ref{thm:correctness_uniform_orth}},
\end{align}
by
Lemma~\ref{lem:procrust} below.

By our choice of $R$, the components of $S_R$ satisfy 
$\cum_4(S_{i,R}) \geq \Delta$ and 
$\E\, S_{i,R}^4 \leq  2 R^4$ (via \eqref{eqn:damping-moment-bound} and our choice $C_1=1/2$). 
Hence $M_4 \leq 2 R^4$. 
The latter
bound via \eqref{eqn:EhE} gives $\E[\inner{X}{u}^4] \le 2 (1 + (n+1)^2 \epsilon_{\ref{thm:correctness_uniform_orth}}) R^4$ for all $u \in \S^{n-1}$.

Finally, Lemma~\ref{lem:Sigma-estimates} 
with \eqref{eqn:expansion-Sigma} \
and simple estimates give 
$
\norm{\Psi_{\hX_R} - \Psi_{X_R}}_F \leq O(n^4 R^4 \epsilon_{\ref{thm:correctness_uniform_orth}}^{1/2}).
$

We are now ready to apply Theorem~\ref{thm:ICA-robust} with 
$\epsilon_{\ref{thm:ICA-robust}} = O(n^4 R^4 \epsilon_{\ref{thm:correctness_uniform_orth}}^{1/2})$ and $\delta_{\ref{thm:ICA-robust}}=\delta/2$. 
This gives that
Fourier PCA produces output $b_1, \ldots, b_n$ such that there are signs $\alpha_i = \pm 1$ and permutation $\pi:[n] \to [n]$ such that 

\begin{align} \label{eqn:application-ICA-robust}
    \norm{A_i - \alpha_i b_{\pi(i)}} \le  O(n^4 R^4 \epsilon_{\ref{thm:correctness_uniform_orth}}^{1/2}) \left(\frac{R^4}{\delta \Delta}\right)^5,
\end{align}
with $\poly(n, 1/\Delta, R, 1/R, 1/\epsilon_{\ref{thm:correctness_uniform_orth}}, 1/\delta)$ sample and time complexity. 
Choose $\epsilon_{\ref{thm:correctness_uniform_orth}}$ so that the RHS of \eqref{eqn:application-ICA-robust} is $\epsilon$.

The number of samples and
time needed for orthogonalization is 
$\poly_\gamma(n, M, 1/s_m, S_M, 1/\epsilon_{\ref{thm:correctness_uniform_orth}}, 1/\delta)$. 
Substituting the value of $\epsilon_{\ref{thm:correctness_uniform_orth}}$ the previous bound becomes 
$\poly_\gamma(n, M, 1/s_m, S_M, 1/\Delta, R, 1/R,1/\epsilon, 1/\delta)$. 
The probability of error, coming from the applications of Theorem~\ref{thm:correctness_uniform_orth} 
and \ref{thm:ICA-robust} is at most 
$\delta/2 +\delta/2 = \delta$. 
\end{proof}

\begin{lemma}\label{lem:procrust}
Let $\hE \in \R^{n\times n}$ be such that $\norm{\hE^T\hE - I}_2 \leq \epsilon$. Then there exists a unitary matrix $E \in \R^{n\times n}$ such that $\norm{E-\hE}_2 \leq \epsilon$. 
\end{lemma}
\begin{proof}
This is related to a special case of the so-called orthogonal Procrustes problem \cite[Section 12.4.1]{MR1417720}, where one looks for a unitary matrix $E$ that minimizes $\norm{E-\hE}_F$. A formula for an optimal $E$ is $E = UV^T$, where $U \Sigma V^T$ is the singular value decomposition of $\hE$, with singular values $(\sigma_i)$. Although we do not need the fact that this $E$ minimizes $\norm{E-\hE}_F$, it is good for our purpose:
\[
\norm{\hE - E}_2 = \norm{U \Sigma V^T - U V^T}_2 = \norm{\Sigma - I}_2 = \max_i \abs{\sigma_i -1}.
\]
By our assumption
\[
\norm{\hE^T \hE - I}_2 = \norm{V \Sigma^2 V^T - I}_2 = \norm{\Sigma^2 - I}_2 = \max_i \abs{\sigma_i^2 -1} = \max_i (\sigma_i + 1) \abs{\sigma_i - 1} \leq \eps.
\]
This implies $\max_i \abs{\sigma_i -1} \leq \eps$. The claim follows.
\end{proof}
\fi
\myparagraph{Acknowledgements.} 
The problem considered here first came to our attention during discussions with Santosh Vempala. 
We also thank him for some early discussions. This material is based upon work supported by the National Science Foundation under Grants No. 1350870 and 1422830.

\bibliographystyle{IEEEtran}
\bibliography{IEEEabrv,ICA_bibliography}
\end{document}